\documentclass[twoside]{article}

\usepackage[accepted]{aistats2026}
\usepackage{url}            % simple URL typesetting
\usepackage{booktabs}       % professional-quality tables
\usepackage{amsfonts}       % blackboard math symbols
\usepackage{nicefrac}       % compact symbols for 1/2, etc.
\usepackage{microtype}      % microtypography
\usepackage{xcolor}         % colors
\usepackage{amsmath}
\usepackage{amssymb}
\usepackage{mathtools}
\usepackage{amsthm}
\usepackage{wrapfig}
\usepackage{caption}
\usepackage{comment}
\usepackage{subcaption}  % for subfigures
\DeclareMathOperator*{\argmin}{argmin}
\DeclareMathOperator*{\Clip}{Clip}
\usepackage{mathtools}

\theoremstyle{plain}
\newtheorem{theorem}{Theorem}[section]
\newtheorem{proposition}[theorem]{Proposition}
\newtheorem{lemma}[theorem]{Lemma}

\theoremstyle{definition}
\newtheorem{definition}[theorem]{Definition}
\newtheorem{assumption}[theorem]{Assumption}
\theoremstyle{remark}

\usepackage{algorithm}
\usepackage{algorithmic}
\usepackage[colorlinks=true, linkcolor=cyan, citecolor=cyan, urlcolor=cyan]{hyperref}

\usepackage[style=authoryear, maxcitenames=1, mincitenames=1, maxbibnames=99, uniquelist=false, hyperref=true]{biblatex}
\renewcommand*{\nameyeardelim}{\addcomma\space}

\hypersetup{
    colorlinks=true,        % Enable colored links instead of boxed links
    linkcolor=cyan,         % Color of internal links (e.g., table of contents)
    citecolor=cyan,         % Color of citation links
    urlcolor=cyan,          % Color of URLs
    pdfborder={0 0 0},      % Remove boxes around links
    pdfborderstyle={/S/U/W 1} % Underline links with width 1
}

\DefineBibliographyStrings{english}{
  andothers = {}  % Suppress "et al."
}

\AtEveryCite{%
  \DeclareNameAlias{default}{family}  % Use only the surname in citations
}

\DeclareFieldFormat{citehyperlink}{%
  \href{#1}{\textcolor{blue}{#1}}}

\renewbibmacro*{cite}{%
  \ifnameundef{labelname}
    {\usebibmacro{cite:label}%
     \setunit{\nametitledelim}%
     \usebibmacro{cite:labelyear+extrayear}}
    {%
      \printtext[bibhyperref]{%
        \printnames{labelname}%
        \setunit{\nameyeardelim}%
        \printfield{labelyear}%
        \printfield{extrayear}%
      }%
    }%
}

\begin{document}



\twocolumn[

\aistatstitle{Adaptive Memory Momentum via a Model-Based Framework for Deep Learning Optimization}

\aistatsauthor{ Kristi Topollai \And Anna Choromanska }

\aistatsaddress{ New York University \And  New York University } ]

\begin{abstract}
The vast majority of modern deep learning models are trained with momentum-based first-order optimizers. The momentum term governs the optimizer's memory by determining how much each past gradient contributes to the current convergence direction. Fundamental momentum methods, such as Nesterov Accelerated Gradient and the Heavy Ball method, as well as more recent optimizers such as AdamW and Lion, all rely on the momentum coefficient that is customarily set to $\beta = 0.9$ and kept constant during model training, a strategy widely used by practitioners, yet suboptimal. In this paper, we introduce an adaptive memory mechanism that replaces constant momentum with a dynamic momentum coefficient that is adjusted online during optimization. We derive our method by approximating the objective function using two planes: one derived from the gradient at the current iterate and the other obtained from the accumulated memory of the past gradients. To the best of our knowledge, such a proximal framework was never used for momentum-based optimization. Our proposed approach is novel, extremely simple to use, and does not rely on extra assumptions or hyperparameter tuning. We implement adaptive memory variants of both SGD and AdamW across a wide range of learning tasks, from simple convex problems to large-scale deep learning scenarios, demonstrating that our approach can outperform standard SGD and Adam with hand-tuned momentum coefficients. Finally, our work opens doors for new ways of inducing adaptivity in optimization.
\end{abstract}

\section{INTRODUCTION}

Stochastic Gradient Descent (SGD)~\parencite{bottou-98x} and its variants~\parencite{sutskever2013importance,kingma2014adam} are widely used for training deep learning models due to their simplicity and efficiency. Many popularly used first-order optimizers rely on Heavy Ball Momentum, commonly defined as:
\begin{align}
    d_{t+1} = \beta d_t + (1 - \beta) \nabla f(x_t), 
    \;\; x_{t+1} = x_t - \eta d_{t+1}, \label{eq:momentum_main}
\end{align}
where \( \eta \) is the learning rate, \( x_t \) denotes model parameters at iteration \( t \), \( f \) is the loss function, and \( \beta \) is the momentum coefficient. Momentum methods augment the current gradient with an exponentially weighted moving average of past gradients, where \( \beta \) determines the optimizer's ``memory'', that is, how much past gradients influence the update direction.

In deterministic settings, momentum methods can provably accelerate convergence under mild assumptions~\parencite{polyak1964some, nesterov1983method}. Achieving such acceleration in practice with Heavy Ball (HB) momentum~\parencite{polyak1964some} or Nesterov Accelerated Gradient (NAG) ~\parencite{nesterov1983method} requires carefully tuning \( \eta \) and \( \beta \), or using time-dependent schedules~\parencite{nesterov1983method}. However, in non-convex or stochastic settings, these accelerated rates are not guaranteed. Surprisingly, despite the widespread adoption of momentum methods in deep learning due to their empirical effectiveness, theoretical analyses offer no justification of why they show empirical gains: under similar assumptions, stochastic Heavy Ball (HB) momentum achieves at best the same convergence rate as plain SGD, but no better. This disconnect between theory and practice is striking, especially given the near-universal use of a fixed momentum coefficient, typically \( \beta = 0.9 \), across models, datasets, and optimization setups.

But is this fixed choice really optimal? Intuitively, it seems unlikely that a single value of \( \beta \) should work equally well throughout the whole training process and across different data sets and models. Instead, we ask: can momentum adapt over time to better match the optimization landscape? In this work, we propose a simple yet principled answer, a time-varying momentum coefficient that evolves with the optimization process. We refer to this mechanism as \textit{adaptive memory}, as it dynamically adjusts the extent to which the optimizer relies on past gradients at each step.

To derive this adaptive coefficient, we take inspiration from model-based optimization techniques~\parencite{asi2019stochastic,davis2019stochastic} and the proximal bundle method~\parencite{kiwiel2006methods}. In this framework, the objective function \( f(x) \) is approximated by a surrogate model \( f^m_t(x) \), and at each step the following problem has to be solved:
\begin{equation}
\label{model_based1}
    x_{t+1} \in \argmin_{x} f^m_t(x) + \frac{1}{2\eta} \|x - x_t\|^{2}.
\end{equation}
We extend this framework by constructing \( f^m_t(x) \) from two planes: one from the current gradient \( \nabla f(x_t) \) and one from the previous descent direction $\frac{1}{\eta}\left(x_{t-1} - x_t\right)$ which encodes the accumulated momentum. Under this model, the proximal step in Equation~\eqref{model_based1} yields the update:
\begin{equation}
    x_{t+1} = x_t - (1 - \beta_t)\eta \nabla f(x_t) +   \beta_t \left(x_{t} - x_{t-1}\right) ,
\end{equation}
where \( \beta_t \) is an \textit{adaptive} momentum coefficient computed in closed form from the model. This gives rise to our proposed \textit{adaptive memory momentum} schemes. \textbf{Our contributions are summarized as follows:}

 \textbf{i) Motivation.} We begin with a simple yet revealing empirical demonstration: using a fixed momentum coefficient can lead to suboptimal convergence, even when finely tuned. Identifying the optimal fixed value of \( \beta \) is often tedious, and more importantly, a static choice fails to adapt to the dynamics of the training process. This observation motivates the need for a time-adaptive momentum scheme.
 
\textbf{ii) Methodology.} We propose a novel approximate model of the loss function that combines two planes, one based on the current gradient and the other on the accumulated momentum. This model leads naturally to a reinterpretation of the proximal model-based update rule as an \textit{adaptive memory} momentum scheme, in which the momentum coefficient evolves throughout optimization. We then incorporate our adaptive memory mechanism into both SGD and AdamW, yielding simple, yet effective, new variants of these optimizers. In each case, the momentum coefficient is updated dynamically based on our model-based formulation.

\textbf{iii) Experimental Validation.} We evaluate our approach across a wide range of settings, from deterministic convex problems to large-scale deep learning tasks. Beyond consistently outperforming fixed-momentum baselines, our adaptive memory methods offer significant benefits in challenging training regimes, particularly at high learning rates (Figure~\ref{fig:ablations_twocol}) and in the early stages of optimization. An important implication of adaptive memory is that it can offer an alternative to learning rate warm-up or scheduling (Figure~\ref{llama}), providing a robust and tuning-free alternative that could simplify the training pipeline of large models.
    
\section{RELATED WORK}
\label{sec:RW}

\subsection{Momentum Methods}

Momentum was first introduced by Polyak~\parencite{polyak1964some} in the form of the Heavy Ball (HB) method, which incorporates the previous iterate into the update:
\begin{equation}
    x_{t+1} = x_t - \eta \nabla f(x_t) + \beta(x_t - x_{t-1}),
\end{equation}
where \( \beta \in [0,1) \) is the momentum coefficient. For \( L \)-smooth and \( m \)-strongly convex quadratic functions, tuned momentum yields accelerated convergence over gradient descent. Nesterov's Accelerated Gradient (NAG)~\parencite{nesterov1983method} achieves similar benefits, and attains the optimal \( O(1/t^2) \) rate for \( L \)-smooth convex functions.

Momentum has since been adapted to stochastic settings~\parencite{tseng1998incremental, ruszczynski1987linearization}, where it provably matches the convergence rate of SGD~\parencite{yan2018unified, sebbouh2021almost}. However, the theoretical justification for its often superior empirical performance under common settings remains incomplete. Despite this gap, HB-style momentum has become a standard component in deep learning optimizers, offering faster convergence and often better generalization~\parencite{sutskever2013importance}. The practical version used in most frameworks is \footnote{PyTorch uses an equivalent momentum update with \( d_{t+1} = \beta d_t + g_t \) and appropriate learning rate scaling.}:
\begin{align}
\label{mom_for_theorem}
    d_{t+1} = \beta d_t + (1 - \beta) g_t, 
    \qquad x_{t+1} = x_t - \eta d_{t+1},
\end{align}
where \( g_t \) is a stochastic gradient such that \( \mathbb{E}[g_t] = \nabla f(x_t) \). The success of momentum has motivated extensive theoretical analysis~\parencite{jelassi2022towards, mai2020convergence, hu2009accelerated, gitman2019understanding, liu2020improved} and is used in numerous optimization algorithms. These include Adam(W)~\parencite{kingma2014adam,Loshchilov2017DecoupledWD} and its variants~\parencite{you2019large, shazeer2018adafactor, pagliardini2024ademamix}, as well as other momentum-based methods like Lion\parencite{chen2024symbolic}, Sophia~\parencite{DBLP:conf/iclr/Liu0HL024}, SOAP~\parencite{vyas2024soap}, MARS~\parencite{yuan2024mars}, and Muon~\parencite{ jordan2024muon, liu2025muon}. Although different, all of these methods rely on a fixed momentum coefficient \( \beta \) and do not adapt it during training.

\subsection{Model-Based Optimization}

We build on the framework of proximal point methods~\parencite{rockafellar1976monotone}, which update iterates via the objective:
\begin{equation}
\label{model_based}
    x_{t+1} \in \argmin_{x} f(x) + \frac{1}{2\eta}\|x - x_t\|^2_2.
\end{equation}
Since exactly solving this is often intractable, it is common to replace \( f(x) \) with a simpler surrogate \( f^m_t(x) \)~\parencite{asi2020minibatch}, which yields a general model-based optimization framework used in both deterministic and stochastic settings~\parencite{davis2019stochastic, asi2019stochastic, chadha2022accelerated}.
Gradient Descent, SGD~\parencite{robbins1951stochastic}, and Subgradient Descent~\parencite{polyak1987introduction} can all be recovered by linearizing \( f \) around \( x_t \) and substituting appropriate gradient or subgradient terms. Second-order updates such as Newton's method arise from quadratic approximations. More sophisticated models, such as cutting plane bundles~\parencite{kelley1960cutting}(Equation~\ref{cutting_plane}), lead to the Proximal Bundle Method~\parencite{kiwiel1983aggregate}, useful in non-smooth optimization:
\begin{equation}
    \label{cutting_plane}
    f^m_t(x) = \max_{i=1,\dots,T} \big(f(x_i) + \langle g_i, x - x_i \rangle \big),
\end{equation}
where $g_i$ is a subgradient that defines a cutting plane at $x_i$. Model-based approaches have also been used to derive adaptive learning rates, such as the Polyak step size~\parencite{polyak1987introduction,loizou2021stochastic}, and more recently, adaptive learning rates for momentum gradient descent~\parencite{schaipp2023momo, wang2023generalized, oikonomou2024stochastic}. Inspired by these ideas, we instead use a model-based approximation of the loss to derive adaptive momentum coefficients.

\subsection{Adaptive Methods}

In both stochastic and non-stochastic optimization, choosing effective values for the learning rate and momentum coefficient is critical. In smooth and (strongly) convex settings, their optimal values depend on the condition number of the function. Consequently, a significant line of research focuses on adaptively estimating the strong convexity and smoothness constants to tune these parameters for various accelerated methods~\parencite{malitsky2019adaptive,barre2020complexity,saab2022adaptive}.

In stochastic optimization, most work focuses on adaptive learning rates. Methods such as stochastic Polyak step sizes~\parencite{loizou2021stochastic,orvieto2022dynamics} adapt the step size based on the suboptimality gap \( f(x_t) - f^* \), while others rely on distance-to-optimum estimates~\parencite{defazio2023learning,ivgi2023dog}. In deep learning, adaptive diagonal pre-conditioners, used in AdaGrad~\parencite{duchi2011adaptive}, Adam~\parencite{kingma2014adam}, and related methods~\parencite{hinton2012neural,DBLP:conf/iclr/Liu0HL024}, are standard for per-parameter step size adjustment. Beyond adaptive schemes, fixed learning rate schedules~\parencite{smith2017cyclical,DBLP:conf/iclr/LoshchilovH17} and warm-up strategies~\parencite{goyal2017accurate} are widely used, particularly in large-scale deep learning. Warm-up, which gradually increases the learning rate at the start of training, is now common in training LLMs \parencite{wortsman2023small}. However, it is not truly adaptive and requires manual tuning of hyperparameters such as warm-up duration, making it brittle in long or infinite horizon training regimes where the total number of steps may be unknown. Consequently, there is a growing interest in methods that alleviate these shortcomings~\parencite{kalra2024warmup,kosson2024analyzing}.

In contrast, adapting the momentum coefficient has received little attention. Heuristic schedules~\parencite{sutskever2013importance,9746839} increase or decrease momentum over time, but have not been adopted in practice. Restart-based methods~\parencite{giselsson2014monotonicity,o2015adaptive} propose clearing the momentum buffer when certain technical conditions are met 
\parencite{o2015adaptive} or according to pre-defined schedules~\parencite{wang2022scheduled} and, while effective, they often require handcrafted restart rules. Conjugate gradient (CG) and nonlinear CG~\parencite{hager2006survey} also adapt the weight on past directions and can be viewed as momentum methods with dynamic coefficients. However, they rely on line search to select learning rates. In this work, we focus solely on adapting the momentum coefficient \( \beta \), without also requiring learning rate adaptation.

\section{METHOD}
\label{sec:Method}

\subsection{A Simple Motivation}

We begin by analyzing the deterministic setting and, before introducing our proposed method, we present a motivational example to highlight two key limitations of using a fixed, constant momentum coefficient. In Figure~\ref{fig:quadratic}, we consider an unconstrained quadratic optimization problem, $\min\{x^T Ax +b^T x + c\}$, and plot the objective gap $f(x_{t})-f^{*}$, where $f^*$ is the minimum, for various fixed momentum coefficients. The results reveal that a fixed momentum is inherently suboptimal, consistently being outperformed by an adaptive strategy. Furthermore, optimization around the optimal value of the momentum coefficient, $\beta^{*}$, is highly unstable, i.e., a slightly different coefficient than the optimal radically degrades the performance.

\begin{figure}[h]
    \centering
    \includegraphics[width=0.75\linewidth]{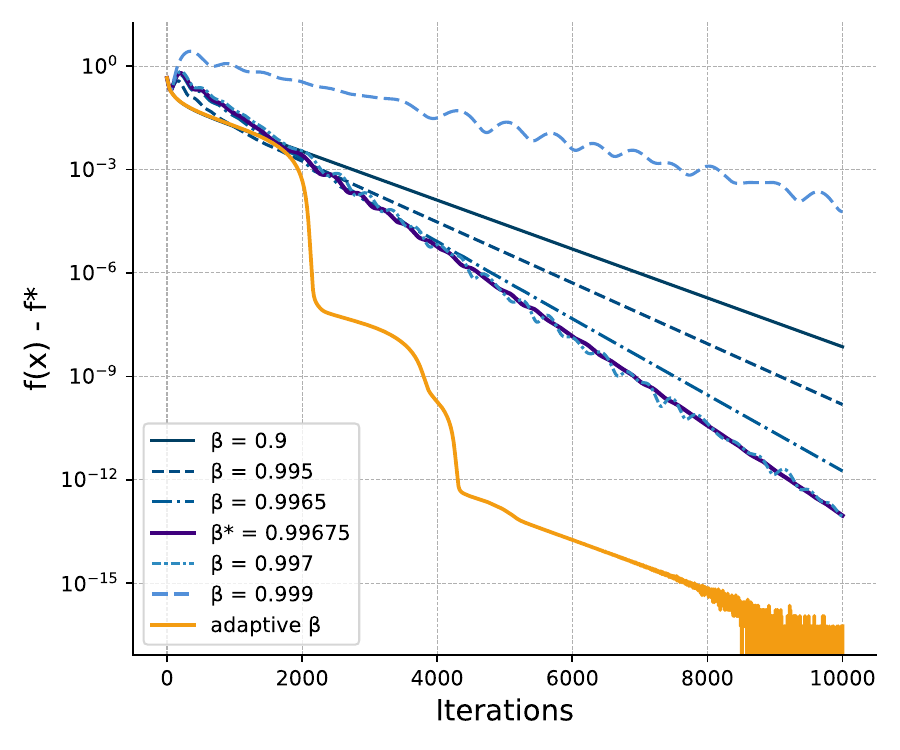}
    \vspace{-1em}
    \caption{The Heavy-ball method with a fixed $\beta$ vs our scheme. We plot the error $f(x_t)-f^*$ over  $t$.}
    \label{fig:quadratic}
    \vspace{-1em}
\end{figure}

\subsection{Approximate Cutting Planes Framework}
Inspired by bundle methods~\parencite{kiwiel2006methods}, our approach for deriving an adaptive coefficient $\beta$ for momentum methods is based on the following model of the function \( f \):
\begin{align}
    f^m_t(x) = \max \Big\{ 
    f\left(x_t\right) &+ g_t^{\top}\left(x - x_t\right),\\ 
    &\hat{f}\left(x_t\right) + \frac{1}{\eta}\left(x_{t-1} - x_t\right)^{\top}\left(x - x_t\right) 
    \Big\}\notag,
\end{align}
where $\frac{1}{\eta}\left(x_{t-1} - x_t\right)$ is the previous descent direction, the momentum term, \( g_t = \nabla f(x_t)\), and $\hat{f}\left(x_t\right)$ is the momentum plane bias, which we discuss further in the end of this section. This formulation uses the cutting plane at the current point, defined by the first-order approximation of \( f \) at \( x_t \), to construct a model of the function. To refine this model, we propose incorporating an additional plane with a slope determined by the direction $\frac{1}{\eta}\left(x_{t-1} - x_t\right)$.

Furthermore, inspired by \parencite{sppam, smod}, we introduce an extra regularization term and regularization parameter $\lambda$ to control the extent to which the descent direction aligns with the previous descent direction $\frac{1}{\eta}\left(x_{t-1} - x_t\right)$. Adding this regularization term to the proximal model-based objective leads to the following update at each step:
\begin{align}
\label{model_based2}
    x_{t+1} \in \argmin_{x} f^m_t(x) &+ \frac{\lambda +1}{2\eta} \|x - x_t\|^2_2 \\&+ \frac{\lambda}{\eta}\langle x_{t-1} - x_t, x-x_{t}\rangle\notag.
\end{align}
The piecewise nature of the truncated model in \( f^m_t(x) \) has a key property, when used in Equation~\eqref{model_based2}, as detailed in the Supplement Section~\ref{sec:derivation} along with all the derivations, the minimizer can be expressed in the following heavy-ball update:
\begin{equation}
x_{t+1} = x_t - \frac{\eta}{\lambda+1} \left(
 a_1 \nabla f(x_t) + (a_2+\lambda)\tfrac{x_{t-1}-x_t}{\eta}
\right).
\end{equation}
where \( a_1, a_2 \in [0,1] \), with \(a_1+a_2=1\), are the dual variables associated with Equation~\eqref{model_based2}. By setting \( a_2 = \beta \) and \( a_1 = 1-\beta \), and  defining the velocity vector $d_t = \frac{1}{\eta}\left(x_{t-1} - x_t\right)$, these variables are determined by solving the quadratic program:
\begin{align}
    \underset{0 \leq \beta \leq 1}{\operatorname{max}} \Bigg\{ 
    -\frac{\eta}{2(\lambda+1)} \Big\| (1 &- \beta)g_t + \beta d_t + \lambda d_t \Big\|_2^2\\
    &+ (1 - \beta) f(x_t) 
    + \beta \hat{f}(x_t) \Bigg\}\notag .
\end{align}
The solution to which, is given by:
\begin{equation}
\beta_t^* = \Clip_{[0,1]}\Bigg(\frac{\frac{(\hat{f}(x_t)- f(x_t))(\lambda+1)}{\eta}  - \left\langle d_t-g_t , g_t+\lambda d_t\right\rangle}{\left\|d_t-g_t\right\|^2_{2}}\Bigg).
\end{equation}
Similar to aggregation in bundle methods, the computed \(\beta_t^*\) is also used to construct the approximation plane for the next iteration. This approach allows us to view the proximal step in Equation~\eqref{model_based2} as a momentum method with an adaptive momentum coefficient:
\begin{align}
    \label{adaptivememory}
    d_{t+1} = \frac{\beta_t^*+\lambda}{1+\lambda} d_t + \frac{1-\beta_t^*}{1+\lambda} g_t, \qquad
    x_{t+1} = x_t - \eta d_{t+1}.
\end{align}
The scaling factor \(\lambda+1\) is introduced in order to define the momentum vector as a convex combination of the gradient and the current momentum. This ensures consistency with the current definition of the momentum gradient descent of Equation~\eqref{eq:momentum_main}. The resulting algorithm is captured in Algorithm~\ref{alg:unified}.

\begin{figure}[h]
    \centering
    \includegraphics[width=\linewidth]{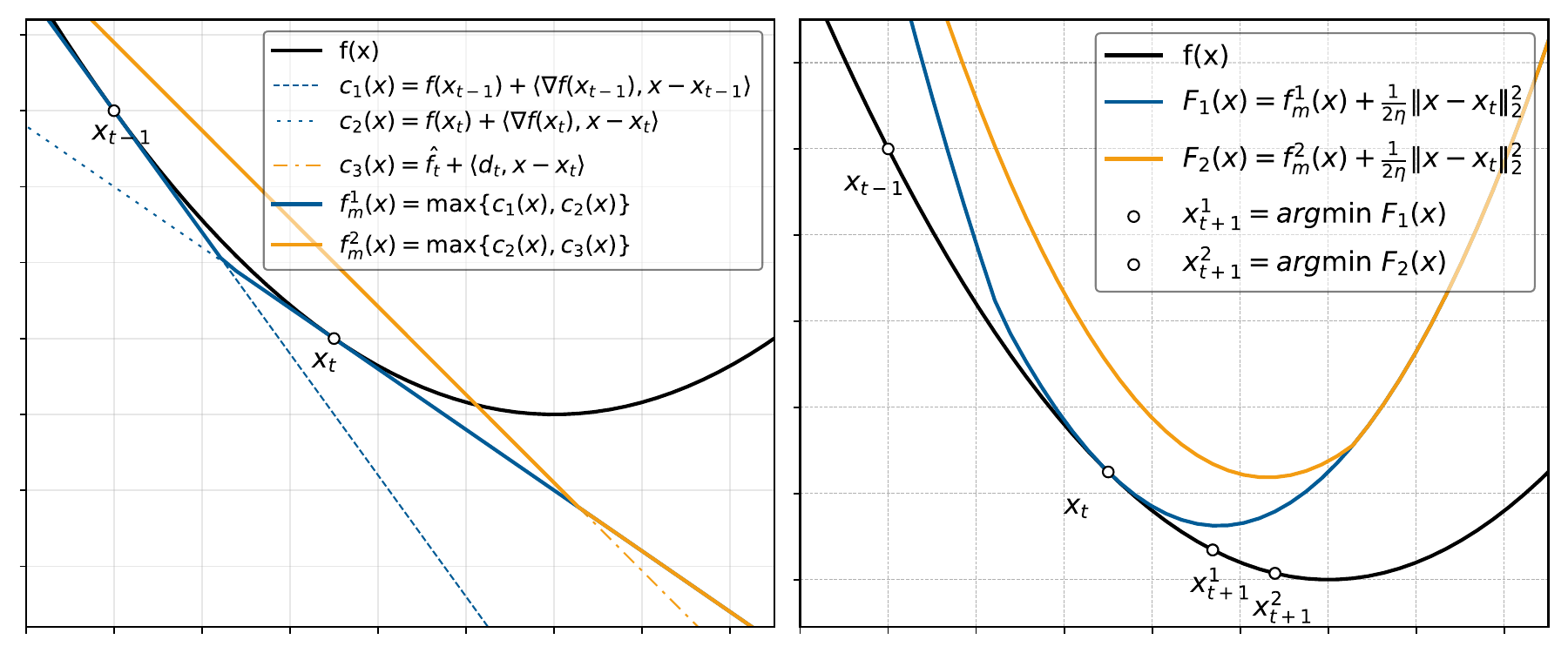}
    \caption{On the left, blue represents a typical cutting planes model, while yellow depicts our model, which may overestimate the function at 
$x_t$. On the right, we observe that with our model, the solution to Equation~\ref{model_based} lies closer to $f^{*}$, thus overestimation can get closer to the optimum.}
    \label{fig:model_example}
    \vspace{-0.1in}
\end{figure}

Regarding the definition of the momentum plane in the model \( f^m_t \), aggregation in proximal bundle methods typically sets \( \hat{f}(x_{t+1}) = f^m_t(x_{t+1}) \leq f(x_t) \), for the next iteration, leveraging the property that \( d_{t+1} \in \partial f^m_t(x_{t+1}) \) \parencite{kiwiel1983aggregate}. Although underestimating the true value of the objective ensures convergence and is used extensively in the literature on bundle methods, our experiments reveal that relying on overestimation and using \( \hat{f}(x_t) = f(x_{t-1}) \) instead (note that it is possible to have \( \hat{f}(x_t) > f(x_t)\)) results in faster convergence. An intuitive illustration of this mechanism is provided in Figure~\ref{fig:model_example}. This is a practical and efficient choice, shown in our experiments to yield strong empirical performance without added overhead. To support our argument, a detailed discussion is provided in the Supplement Section~\ref{sec:overestimating}, where we prove that for quadratics, minimizing \(\|x_{t+1} - x^*\|_2^2\) under Equation~\ref{mom_for_theorem} with respect to $\beta$ implies \(\hat{f}(x_t) > f(x_t)\) for all \(t\).

\subsection{Pre-conditioning and Decoupled Weight decay}
The flexibility of the model-based approach allows seamless extension to state-of-the-art optimizers. Algorithms like Adam~\parencite{kingma2014adam} and AdaGrad~\parencite{duchi2011adaptive} use diagonal pre-conditioners to scale updates. To incorporate these into our Adaptive Memory framework, we replace the Euclidean metric with one induced by the pre-conditioner \( P_t \), where \( \langle x, y \rangle_{P_t} = x^\top P_t y \) for symmetric positive definite \( P_t \). Weight decay can be accounted for by modeling the regularized function \( f_{\mu\|\|}(x) = f(x) + \frac{\mu}{2} \|x\|_2^2 \). However, optimizers like AdamW~\parencite{Loshchilov2017DecoupledWD} implement decoupled weight decay, which is handled separately. We can incorporate a preconditioner and decoupled weight decay by solving at each step:
\begin{align}
\label{model_based4}
    x_{t+1} \in \argmin_{x} f^m_t(x) &+ \frac{1}{2\eta} \|x - x_t\|^2_{\tilde{P}_t}\\&+ \lambda \langle d_t, x - x_t \rangle_{P_t} + \frac{\mu}{2} \|x\|^2_{\tilde{P}_t}\notag,
\end{align}
where \( \tilde{P}_t = (I + \lambda P_t)P_t \) is introduced to appropriately scale the momentum term. Solving this optimization problem yields the following update rules:
\begin{align}
    d_{t+1} &= (\lambda P_t + I)^{-1} \big((\beta_t^* I + \lambda P_t) d_t + (1 - \beta_t^*) g_t \big),\notag\\
    x_{t+1} &= \frac{1}{1 + \mu \eta} \big(x_t - \eta P_t^{-1} d_{t+1} \big).\notag\\
    \beta_t^* &= \Clip_{[0, 1]} \Bigg(\frac{\mu \langle x_t, g_t - d_t \rangle}{\|d_t - g_t\|^2_{\tilde{P}_t^{-1}}}\\&+ \frac{\frac{(1 + \mu \eta)(\hat{f}(x_t) - f(x_t))}{\eta} - \left\langle d_t - g_t, g_t + \lambda P_t d_t \right\rangle_{\tilde{P}_t^{-1}}}{\|d_t - g_t\|^2_{\tilde{P}_t^{-1}}} \Bigg)\notag .
    \label{final_beta}
\end{align}
By using the Adam pre-conditioner: $P_t = (1 - \beta_1^t) \text{diag} \big( \epsilon + \sqrt{\frac{v_t}{1 - \beta_2^t}} \big)$ where $v_t = \beta_2 v_{t-1} + (1 - \beta_2) (g_t \odot g_t)$, we derive the AM-AdamW \footnote{for small \( \eta \), \( \frac{1}{1 + \eta \mu} \approx 1 - \eta \mu \) and \( \frac{\eta}{1 + \eta \mu} \approx \eta \)~\parencite{zhuang2022understanding}} optimizer.
This method retains the benefits of AdamW while incorporating momentum adaptivity through the AM framework.

\subsection{Adaptive Memory for Stochastic Gradient Descent}
We found that the most variance-sensitive component in the adaptive momentum formula is the difference \( \Delta f = f(x_{t-1}, s_t) - f(x_t, s_t) \), where \(f(x_t, s_t)\) are stochastic function evaluations using a data batch \( s_t \) sampled at time \( t \). Evaluating the loss at two different points on the same batch is inefficient, while using different batches
 introduces significant noise and instability into the estimate. To address this, we apply two practical modifications detailed in Algorithm~\ref{alg:unified}:
\begin{itemize}
\vspace{-1em}
    \item Clip \( \beta \) to the interval \( [0, \beta_{\max}] \), with \( \beta_{\max} < 1 \) (0.9 in experiments of Section~\ref{experiments}).
    \vspace{-0.1in}
    \item Replace \( f(x_{t-1}, s_t) - f(x_t, s_t) \) with a first-order approximation around \( x_t \): 
    \vspace{-0.1in}
    \[
        \Delta f \approx \nabla f(x_t, s_t)^\top (x_{t-1} - x_t) = \eta g_t^\top d_t,
    \]
\end{itemize}
\vspace{-1em}
We apply similar adjustments to the AdamW setting. While our notation assumes a fixed learning rate \( \eta \), our framework also supports dynamic learning rates. Full algorithmic and implementation details for all Adaptive Memory variants are provided in the Supplement Section~\ref{sec:algo}.

\newcommand{\MGD}[1]{\textcolor{cyan}{#1}}
\newcommand{\MSGD}[1]{\textcolor{orange}{#1}}

\begin{algorithm}[t]
\small
\caption{Adaptive Memory–Momentum. 
\MGD{AM-MGD}, \MSGD{AM-MSGD}}
\label{alg:unified}
\begin{algorithmic}[1]
\STATE \textbf{Initialize:} $\eta, \lambda, x_0, $
  \MGD{$ d_0 = \nabla f(x_0)$}, \;
  \MSGD{$\beta_{\max}, d_0 = g_0$}
\FOR{each iteration $t=1,2,\dots,T$}
  \STATE $g_t =$
    \MGD{$\nabla f(x_t)$}, \;
    \MSGD{$\nabla f(x_t, s_t)$}
  \STATE  
    \MGD{$\beta_t =\displaystyle \frac{(\hat{f}(x_t) - f(x_t))(\lambda + 1) - \eta\langle d_t - g_t,\; g_t + \lambda d_t \rangle}{\eta\|d_t - g_t\|_2^2}$}
  \STATE 
    \MSGD{$\beta_t =\displaystyle \frac{(\lambda + 1) g_t^\top d_t - \langle d_t - g_t,\; g_t + \lambda d_t \rangle}{\|d_t - g_t\|_2^2}$}
  \STATE $\beta_t =$
    \MGD{$\min(\max(\beta_t, 0), 1)$}, \;
    \MSGD{$\min(\max(\beta_t, 0), \beta_{\max,t})$}
  \STATE $d_{t+1} = \frac{\beta_t + \lambda}{1 + \lambda} d_t + \frac{1-\beta_t}{1 + \lambda} g_t$
  \STATE $x_{t+1} = x_t - \eta d_{t+1}$
\ENDFOR

\end{algorithmic}
\end{algorithm}

\subsection{Convergence Guarantees}
\label{theory}

We provide convergence guarantees for our proposed method in the standard finite-sum setting:
\[
\min_{x \in \mathbb{R}^d} f(x)
=
\frac{1}{n}\sum_{i=1}^n f_{S_i}(x),
\]
where each \(f_{S_i}\) denotes the loss associated with a mini-batch \(S_i\), and the goal is to minimize the average loss \(f\). At iteration \(t\), let $g_t := \nabla f_{S_t}(x_t)$,
where \(S_t\) denotes the sampled mini-batch. We proceed under the following assumptions.

\begin{assumption}[Smoothness]
\label{Smoothness}
The function \(f\) is \(L\)-smooth if for all \(x,y \in \mathbb{R}^d\),
\begin{equation}
f(y) \leq f(x) + \nabla f(x)^\top (y-x) + \frac{L}{2}\|y-x\|^2.
\end{equation}
\end{assumption}

\begin{assumption}[Bounded stochastic gradient norm]
\label{growth}
There exists \(G>0\) such that, for all \(x\),
\begin{equation}
\mathbb{E}_{S}\!\left[\|\nabla f_{S}(x)\|^2\right] \le G^2.
\end{equation}
\end{assumption}

We additionally assume that the stochastic gradient is unbiased:
\[
\mathbb{E}[g_t \mid x_t] = \nabla f(x_t).
\]

\vspace{-0.5em}
For our Adaptive Memory Momentum SGD, presented in Algorithm~\ref{alg:unified}, we establish convergence guarantees under standard assumptions. The analysis leverages the structural bound
\[
\beta_t \,\|d_t-g_t\|_2^2 \le \|g_t\|_2^2,
\]
which follows directly from the construction of the adaptive momentum coefficient \(\beta_t\).

\begin{theorem}[Convex Convergence]
\label{thm:convex}
Assume each \(f_{S_i}\) is convex, Assumption~\ref{growth} holds, and
\(\mathbb{E}[g_t\mid x_t]=\nabla f(x_t)\). Let \(x^*\in\arg\min_x f(x)\), and run the method with \(\eta=1/\sqrt{T}\) and \(\beta_{\max}=1/T\). Define
\[
x_T^{(a)}
=
\frac{\sum_{t=0}^{T-1} a_t x_t}{\sum_{t=0}^{T-1} a_t},
\qquad
a_t=\left(1+\frac{1}{T}\right)^{-(t+1)}.
\]
Then
\begin{equation}
\mathbb{E}\bigl[f(x_T^{(a)})-f(x^*)\bigr]
\le
\frac{1}{\sqrt T}
\left(
\|x_0-x^*\|^2
+
\frac{5}{2}G^2
\right).
\end{equation}
\end{theorem}

\begin{theorem}[Non-convex Convergence]
\label{thm:nonconvex-convergence}
Assume \(f\) satisfies Assumption~\ref{Smoothness}, Assumption~\ref{growth} holds, and \(\mathbb{E}[g_t\mid x_t]=\nabla f(x_t)\). Run the method with \(\eta=1/\sqrt{T}\) and \(\beta_{\max}=c\eta\) for some \(0<c<1\). Then
\begin{equation}
\min_{0\le t\le T-1}\mathbb{E}[\|\nabla f(x_t)\|_2^2]
\le
\frac{\mathbb{E}[f(x_0)-f(x_T)] + \tfrac{1+8L}{4}G^2}
{\sqrt{T}(1-c)}.
\end{equation}
\end{theorem}

Our convergence guarantees place AM within the standard landscape of stochastic first-order methods. Under smoothness and bounded stochastic gradient norms, we obtain the classical $O(1/\sqrt{T})$ rate, consistent with existing analyses of SGD and heavy-ball under comparable assumptions~\parencite{yan2018unified,sebbouh2021almost}. The main technical obstacle is that $\beta_t$ is itself a nonlinear function of the current stochastic gradient $g_t$, so the proof must control terms such as $-2\eta \beta_t \langle x_t-x^*,\, d_t-g_t\rangle$, and $\beta_t$ cannot be treated as an external parameter inside expectations.

This is closely related to the difficulty that appears in analyses of adaptive step sizes, where the step size is also chosen from the same stochastic information used in the update, breaking the usual martingale structure unless additional assumptions are imposed~\parencite{loizou2020momentum,orvieto2022dynamics}. In our setting, the key structural relation is $\beta_t\|d_t-g_t\|^2 \le \|g_t\|^2$, which follows directly from the definition of $\beta_t$ and keeps the adaptive correction term controlled. At the same time, as in fixed-momentum and adaptive-step-size methods, the cleanest stochastic guarantees are often obtained when momentum is damped or effectively small~\parencite{wang2023generalized,oikonomou2024stochastic,schaipp2023momo,loizou2020momentum}. Our restriction on $\beta_{\max}$ should be viewed in this same light.

\begin{figure*}[t]
    \centering
    % First figure
    \begin{subfigure}{\linewidth}
        \centering
        \includegraphics[width=\linewidth]{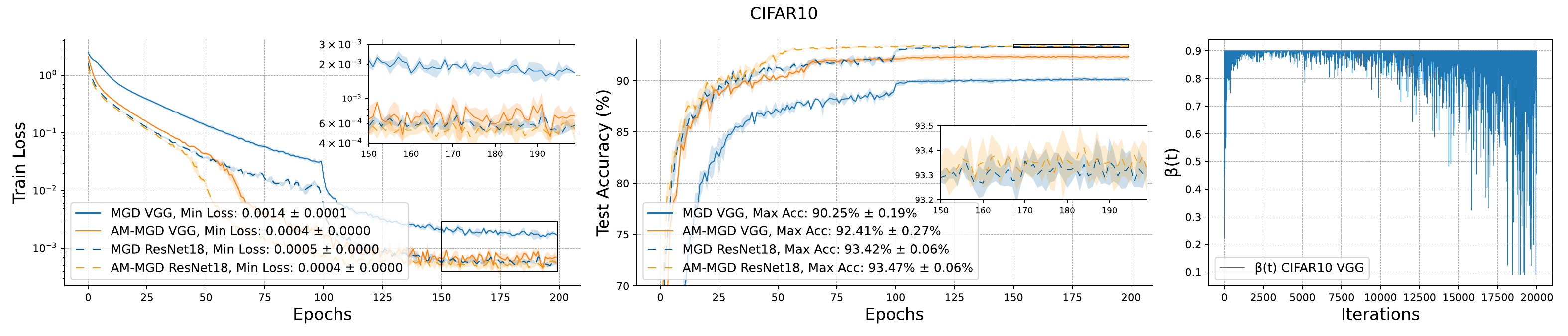}
    \end{subfigure}
    
    \vspace{-.5em} % vertical space between them
    
    % Second figure
    \begin{subfigure}{\linewidth}
        \centering
        \includegraphics[width=\linewidth]{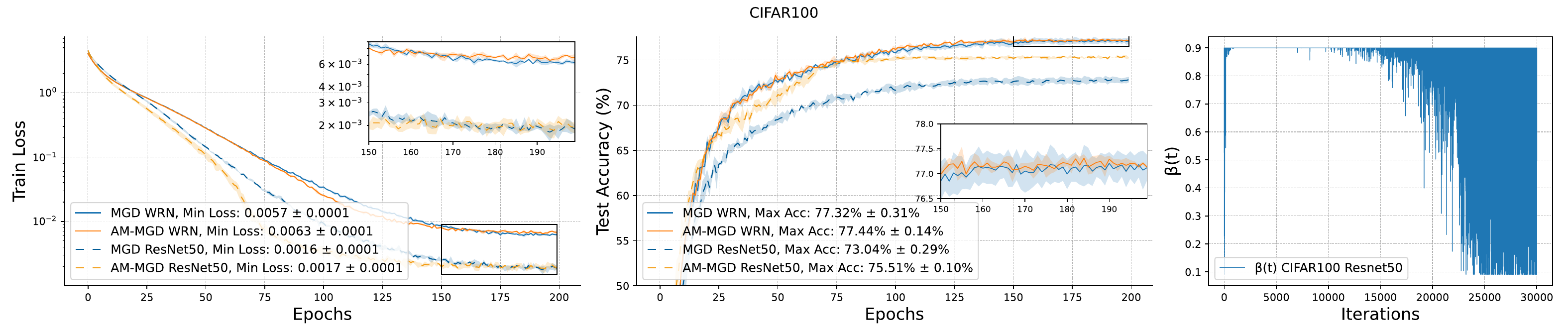}
    \end{subfigure}
    \caption{Train Loss, Test Accuracy and momentum parameter on the image classification experiments, (a) VGG19 and a ResNet18 on CIFAR10 (top), (b) ResNet50 and WideResNet on CIFAR100 (bottom) }
    \label{fig:classification_experiments}
\end{figure*}

\section{EXPERIMENTS}
\label{experiments}
\label{sec:Exp}

\subsection{Convex Problems}
For our convex experiments, we used logistic regression on 9 datasets from the LIBSVM repository~\parencite{CC01a} (4 shown here; full results in the Supplement Section~\ref{sec:conv}). Figure~\ref{fig:logreg} shows that AM-MGD consistently outperforms constant momentum across all datasets. We compare with the commonly used $\beta = 0.9$ and the optimal $\beta^*$, found via grid search to minimize the final loss. Notably, $\beta^*$ can exhibit instability or nonmonotonicity, whereas AM-MGD achieves lower loss without tuning.
\begin{figure}[H]
\centering

% --- Top row: first 2 panels ---
\begin{subfigure}{\linewidth}
    \centering
    \includegraphics[
        width=\linewidth,
        clip,
        trim=0cm 0cm 22.5cm 0cm
    ]{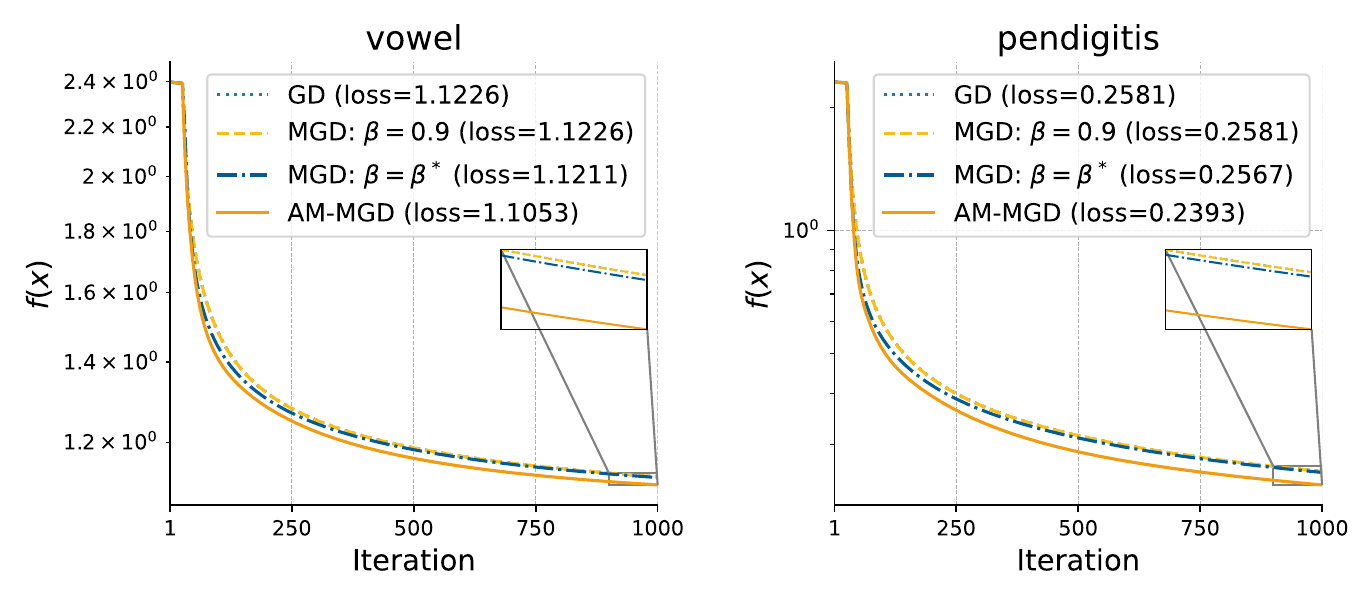}
\end{subfigure}

\begin{subfigure}{\linewidth}
    \centering
    \includegraphics[
        width=\linewidth,
        clip,
        trim=22.8cm 0cm 0cm 0cm
    ]{fixed_figs/logreg_2_zoomed.pdf}
\end{subfigure}

\caption{Logistic Loss over time, AM-MGD (red curve) outperforms fixed momentum on all 4 experiments. GD and MGD with $\beta=0.9$ practically overlap. The $\lambda$ 
was fixed to 0 in all runs.}
\label{fig:logreg}
\end{figure}

\subsection{Image Classification}
\label{imageclassification}

We evaluate AM-MGD on standard image classification benchmarks, including CIFAR-10, CIFAR-100~\parencite{alex2009learning}, and ImageNet~\parencite{russakovsky2015imagenet}, using VGG~\parencite{Simonyan2014VeryDC}, ResNets~\parencite{he2016deep}, and Wide-ResNets~\parencite{zagoruyko2016wide}. For these experiments, we directly adopt the training configuration of MoMo~\parencite{schaipp2023momo}, which likewise considers damped-momentum SGD, and plug our adaptive-momentum update into this setup without performing a separate hyperparameter search for AM-MGD. Thus, all reported improvements are obtained under hyperparameters tuned for the vanilla fixed-momentum baseline, not for our method. Throughout, we keep the AM-specific hyperparameters fixed at \(\beta_{\max}=0.9\) and \(\lambda=0.1\); we do not retune them across datasets or architectures. For MGD, we set PyTorch's dampening parameter to \(0.9\) to ensure a fair comparison at identical learning rates. Finally, the learning-rate ablations in Figure~\ref{fig:ablations_twocol} and in the Supplement Section~\ref{sec:lr_ablation}, show that the relative gains of AM-MGD persist across a broad range of step sizes. Full experimental details are provided in the Supplement Section~\ref{sec:setup_experiment}.

\begin{figure*}[t]
\centering
\includegraphics[width=\textwidth]{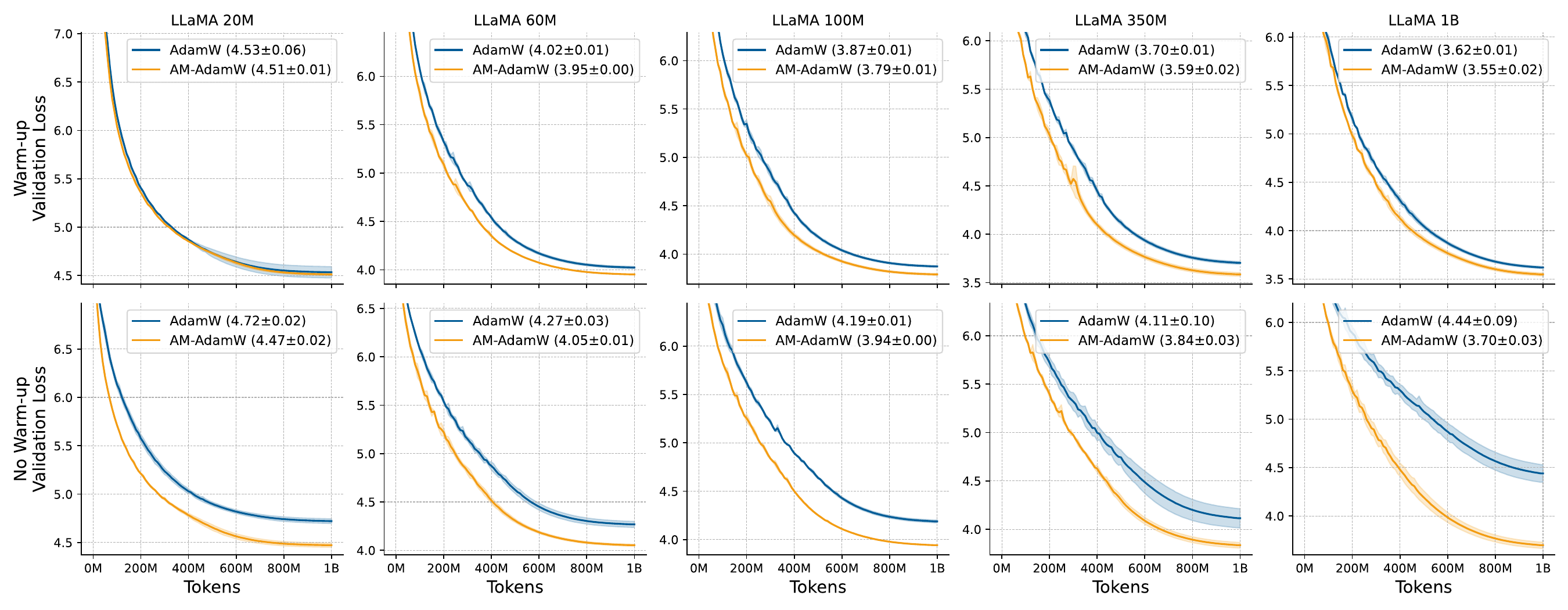} 
\caption{Validation loss curves for LLaMA pretraining on C4 across 5 different model scales.}
\vspace{-1em}
\label{llama}
\end{figure*}

As shown in Figure~\ref{fig:classification_experiments}, AM-MGD consistently outperforms fixed-momentum baselines, yielding both faster optimization and better generalization. This improvement is not only visible in the final accuracy, but also in the training dynamics: across the vision experiments, AM-MGD reaches strong performance substantially earlier, often requiring markedly fewer epochs to attain the same accuracy level as the fixed-momentum baseline. The trajectories of \(\beta_t\) further clarify how this improvement arises. Across the CIFAR experiments, we observe three recurring phases. Early in training, gradients are noisy and the EMA direction \(d_t\) is still poorly estimated, so the misalignment term \(\|d_t-g_t\|_2^2\) is large and \(\beta_t\) remains small, effectively reducing the method closer to SGD and stabilizing the initial phase. As training progresses and gradients become more coherent, the misalignment shrinks, \(\beta_t\) increases, and momentum becomes reliable; this coincides with the main acceleration regime, where we observe the largest gains in convergence speed and learning-rate robustness. Finally, in interpolating settings such as CIFAR-10/100, once the training loss becomes very small, the gradient signal weakens while stochastic noise persists. In this regime, \(\|d_t-g_t\|_2^2\) increases again and \(\beta_t\) exhibits frequent sharp drops toward zero, producing repeated restarts of the momentum buffer. This behaviour is qualitatively consistent with prior observations that momentum can become less beneficial near interpolation~\parencite{defazio2020momentum}. Larger-scale ImageNet experiments in the Supplement Section~\ref{sec:imagenet} show the same overall pattern.

\subsection{Pretraining Large Language Models}
We pretrain LLaMA models~\parencite{Touvron2023LLaMAOA,grattafiori2024llama} of varying scales on the English subset of C4~\parencite{raffel2020exploring}. Rather than training to convergence, these experiments focus on early-phase optimization over \(1\mathrm{k}\) iterations under a fixed token budget. For each model we tune the AdamW baseline over sequence length, batch size, learning rate, and \(\beta_2\), subject to the constraint \(\texttt{batch\_size}\times\texttt{seq\_len}=10^6\), corresponding to a total of \(10^9\) tokens over \(1\mathrm{k}\) steps. We then reuse the same hyperparameters for AM-AdamW, so any performance differences are attributable to momentum adaptation rather than additional tuning. Unlike the vision experiments, where the adaptive momentum coefficient is global, here we compute it separately for each layer to account for layer heterogeneity and to avoid having any single layer dominate the momentum adaptation of the entire model. Each model is trained under two regimes, with and without the standard linear warm-up schedule~\parencite{hagele2024scaling}, in order to study the effect of our method during the earliest phase of training.

The results demonstrate a clear and consistent advantage for AM-AdamW over the standard AdamW optimizer across all model scales. Most notably, even without warm-up, AM-AdamW compares to the performance of AdamW with warm-up. In contrast, AdamW struggles to train reliably without warm-up, particularly at larger scales. These findings highlight two key benefits of our adaptive memory approach: it stabilizes the early stages of training, leading to stronger and more consistent performance, and it shows promise as a hyperparameter-free alternative to warm-up.

\subsection{Ablation Study}
\label{sec:ablation}
To further evaluate our method, we conducted the following three ablation studies on the image classification tasks: ResNet18/50 on CIFAR10/100. 

    \textbf{Ablation on $\mathbf{\lambda}$} Figure \ref{fig:ablations_twocol}a: We examine the robustness of our method with respect to the choice of the hyperparameter \( \lambda \). In nearly all of our experiments, \( \lambda \) is set to 0.1. Notably, all values in the range \( [0.01, 1] \) appear to yield comparable performance. Moreover, when \( \lambda \) becomes too large, the effective range of \( \beta_t \) is reduced due to the scaling factor \( 1+\lambda \) (see Equation~\ref{adaptivememory}). As a result, \( \beta_t \) effectively reaches its upper bound, \( \beta_{\max} = 0.9 \). This explains why, for large values of \( \lambda \), our method exhibits behaviour similar to that of a fixed momentum coefficient.

     \textbf{Ablation on batch size } Figure \ref{fig:ablations_twocol}b: The approximate cutting planes model of the loss used to estimate $\beta_t$ at each iteration is constructed using stochastic gradients. Consequently, a key question arises: how does the stochasticity induced by the batch size influence the overall performance of our method? As expected, increasing the batch size while keeping the learning rate fixed tends to degrade generalization. However, smaller batch sizes, despite leading to more inaccurate models of the loss, do not negatively affect our method’s relative performance gains.

    \textbf{Ablation on learning rate } Figure \ref{fig:ablations_twocol}c: We fix the hyperparameters to the values used in image classification experiments captured in Figure~\ref{fig:classification_experiments} and vary the learning rate. Interestingly, our adaptive memory scheme extends the range of admissible learning rates, enabling convergence even at higher values, unlike the static baseline, which fails to converge as fast with large learning rates.

\begin{figure*}[t]
    \centering
    \begin{minipage}{0.32\textwidth}
        \centering
        \includegraphics[width=\linewidth]{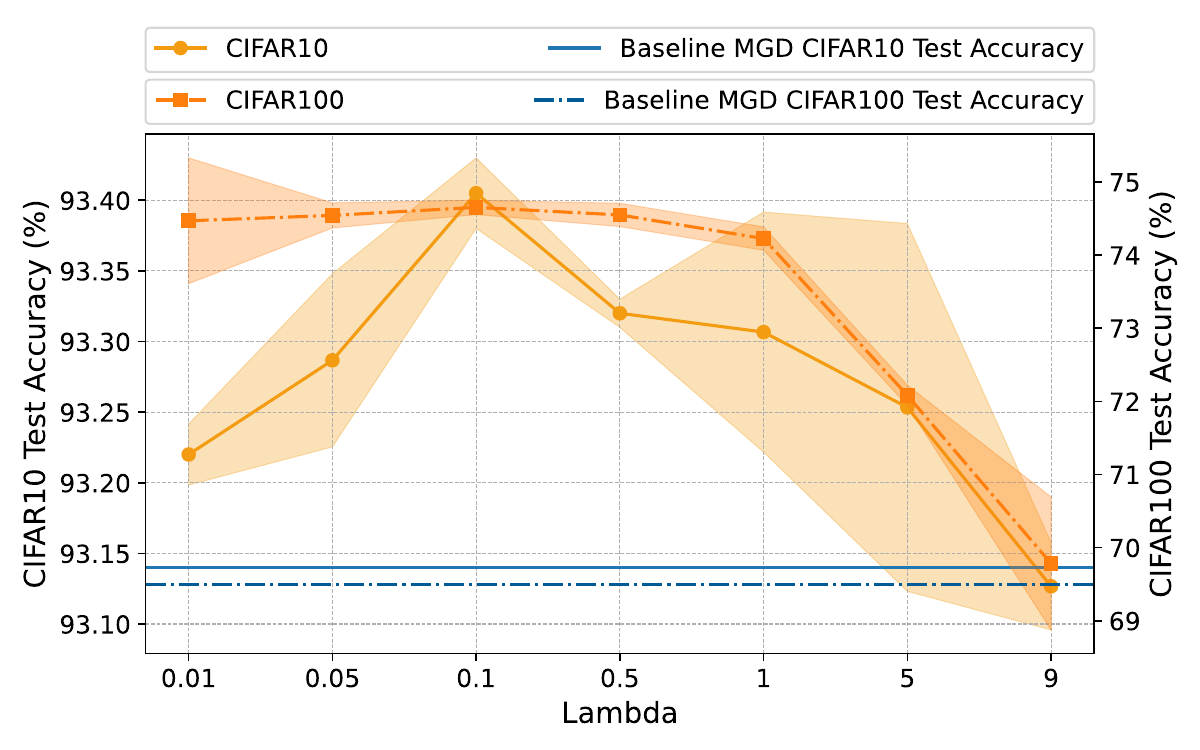} 
    \end{minipage}
    \hfill
    \begin{minipage}{0.32\textwidth}
        \centering
        \includegraphics[width=\linewidth]{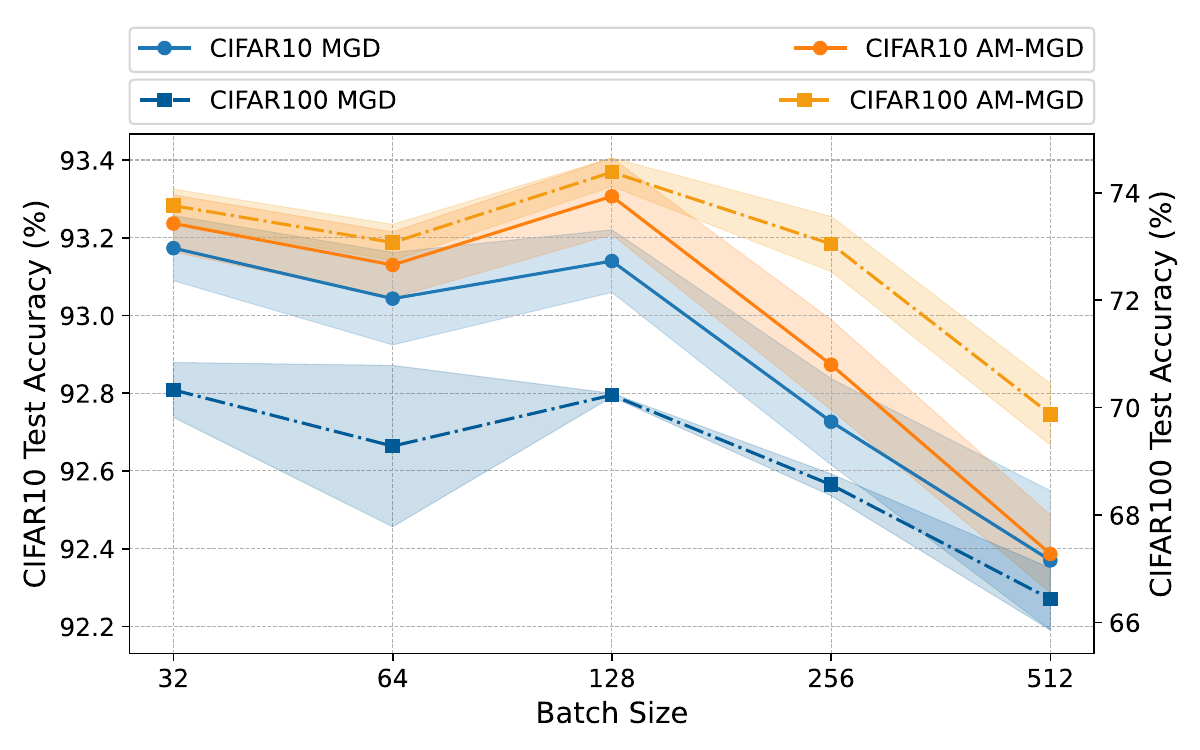} 
    \end{minipage}
    \hfill
    \begin{minipage}{0.32\textwidth}
        \centering
        \includegraphics[width=\linewidth]{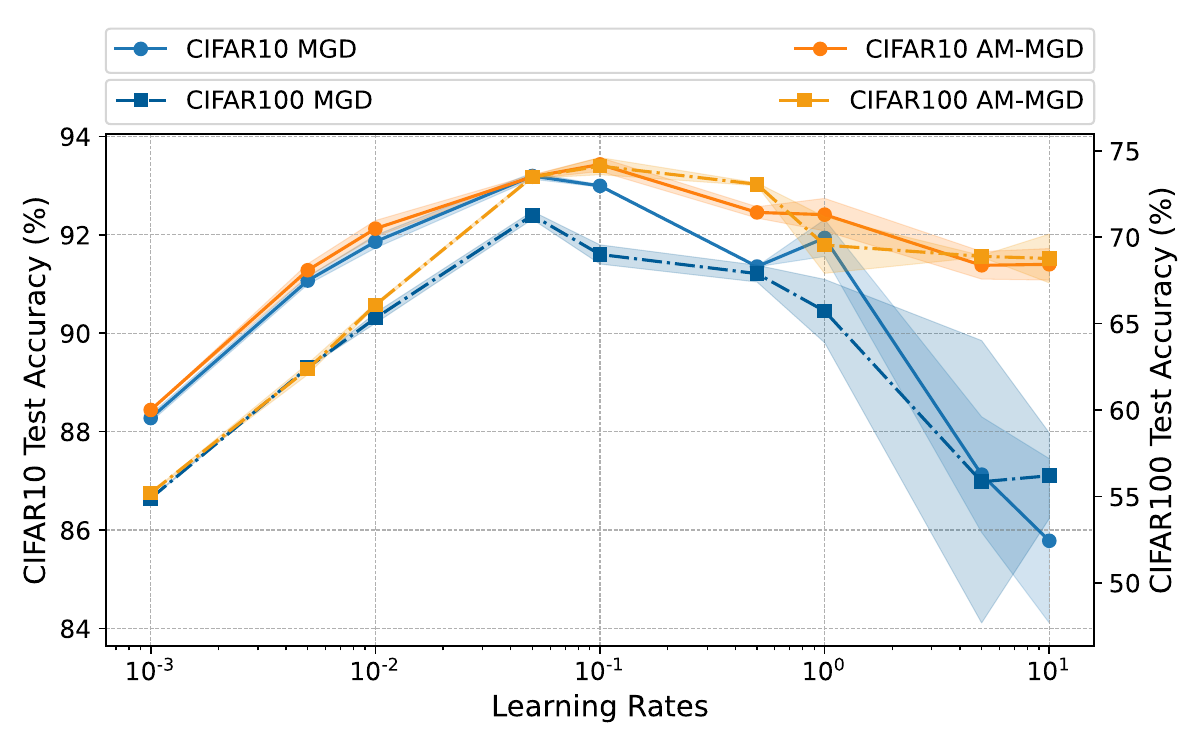}
        \label{fig:imagenet_main}
    \end{minipage}
    \vspace{-1em}
    \caption{Final test accuracy for ResNet18 (CIFAR-10) and ResNet50 (CIFAR-100) under different (left) $\lambda$, (middle) batch size, and (right) learning rate.}
     \vspace{-1em}
    \label{fig:ablations_twocol}
\end{figure*}

\subsection{What matters when adapting $\beta$}

When observing the behaviour of $\beta_t$, we identify two regimes: a high-momentum phase, where $\beta_t$ remains close to its maximum and carries long-term memory, and a low-momentum or ``soft restart'' phase, where $\beta_t$ drops sharply to discard stale information. To better understand these two modes of operation, we explore two variants of adaptive memory: \emph{clipping}, which enforces a lower bound on $\beta_t$, and \emph{restarting},  
which resets momentum when it falls below a threshold.
\begin{equation}
    \beta_t^{\text{Clip}} = \max(\beta_{\text{AM}}, \beta_{\min}),\quad \beta_t^{\text{Reset}} =
\begin{cases}
  0.9, & \beta_{\text{AM}} \geq \theta, \\[0.3em]
  0.1, & \beta_{\text{AM}} < \theta\notag
\end{cases}
\end{equation}
\begin{figure}[H]
    \centering
    \includegraphics[width=\linewidth]{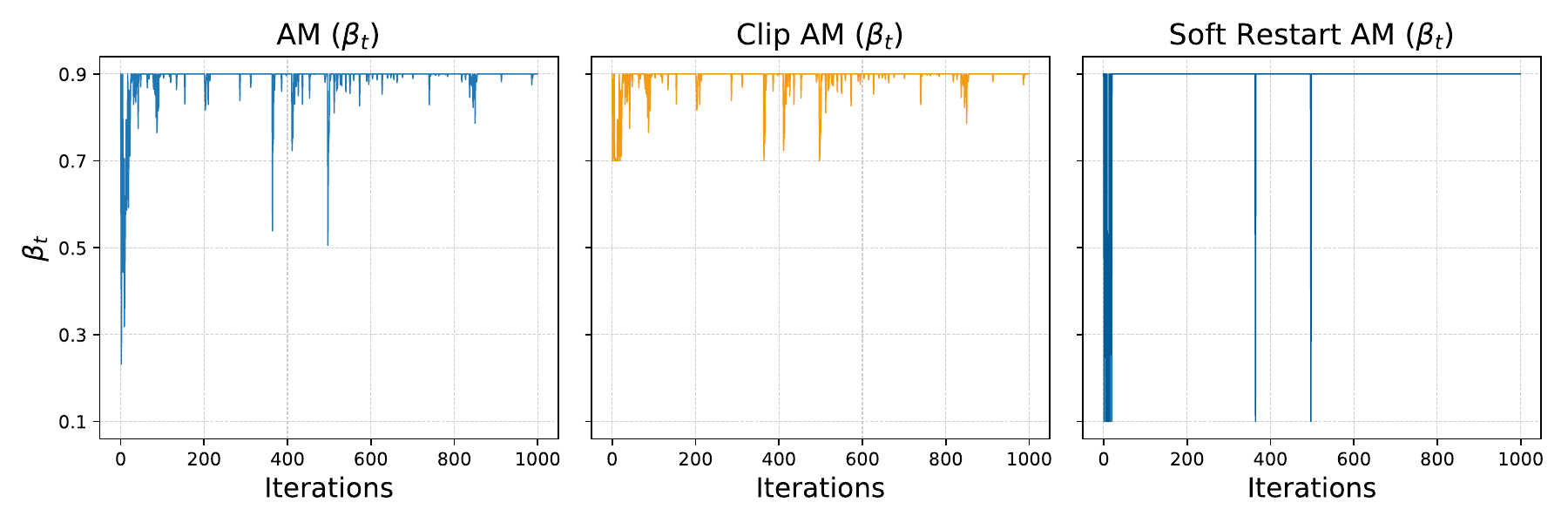}
    \caption{The three AM-derived policies.}
    \label{fig:three_betas}
\end{figure}

We find that early restarts carry much of the benefit by rapidly adapting to unstable initial dynamics as we see in Table~\ref{tab:three_betas}, however, these restarts alone do not explain all of the improvements, indicating that the ongoing adaptation of $\beta_t$ throughout training is also essential.

\begin{table}[H]
\footnotesize
  \centering
  \begin{tabular}{@{}lll@{}}
  \toprule
                  & \textbf{60M} & \textbf{100M} \\ \midrule
  AdamW             & 4.02         & 3.87          \\
  AM-AdamW          & 3.95         & 3.79          \\ \midrule
  AM-AdamW (Clip)   & 3.99         & 3.84          \\
  AM-AdamW(Restart) & 3.97         & 3.82          \\ \bottomrule
  \end{tabular}
  \caption{The three alternative AM optimizers, we set the restarting threshold $\theta = 0.7$.}
  \label{tab:three_betas}
\end{table}

\subsection{Choice of $\beta_{\max}$}

A natural question is how to set the ceiling $\beta_{\max}$ on the adaptive momentum coefficient. To keep the method as tuning-light as possible, in all main experiments we use the baseline-aligned heuristic $\beta_{\max}=\beta$, where \(\beta\) is the fixed momentum of the corresponding vanilla baseline (e.g., \(\beta_1=0.9\) for AdamW/SGD). This provides a simple default recipe and isolates the effect of adaptivity, since gains over the baseline cannot be attributed to a different global momentum value. Intuitively, larger ceilings can help when gradients are coherent, but in static methods they also retain unreliable EMA memory for longer. AM is designed to mitigate this by down-weighting such memory, which helps explain why its advantage over the static baseline becomes larger at higher ceilings. In practice, \(\beta_{\max}\in\{0.9,0.95\}\) works best, while \(\beta_{\max}=0.99\), although still much better than the corresponding static baseline, is somewhat less reliable.

\begin{table}[H]
\centering
\small
\begin{tabular}{@{}ccccc@{}}
\toprule
$\beta_{\max}$ & \multicolumn{2}{c}{60M} & \multicolumn{2}{c}{100M} \\ \midrule
               & AdamW  & AM-AdamW       & AdamW  & AM-AdamW        \\ \midrule
0.95           & 4.387  & \textbf{3.997} & 4.217  & \textbf{3.822}  \\
0.99           & 6.083  & \textbf{4.568} & 6.133  & \textbf{4.474}  \\ \bottomrule
\end{tabular}
\caption{Sensitivity to the momentum ceiling $\beta_{\max}$ on LLaMA pretraining. Lower final validation loss is better.}
\vspace{-1em}
\label{tab:beta_max_sensitivity}
\end{table}

\subsection{Computational Requirements}
In terms of computational cost, AM-MGD and AM-AdamW introduce minimal overhead compared to their non-adaptive counterparts. Specifically, each iteration of AM-MGD incurs an additional \(3N\) FLOPs, and AM-AdamW adds \(6N\) FLOPs, where \(N\) is the number of model parameters. For Transformer-based language models, the total FLOPs per iteration are approximately \(6NB\) \parencite{kaplan2020scaling}, where \(B\) is the number of tokens processed (batch size times sequence length). In our setup, this \(1/B\) overhead translates into a 0.2\% absolute time overhead. Moreover, the memory footprint of our methods is identical to that of the baselines, i.e., AM-MGD and AM-AdamW require no additional memory beyond what is already used by MGD and AdamW, respectively. Detailed wall-clock overhead can be found in the Supplement Section~\ref{sec:computations}.
\section{DISCUSSION AND FUTURE WORK}
\vspace{-0.5em}
\label{Sec:DFW}
We have presented Adaptive Memory (AM), a principled framework that endows momentum-based optimizers with a dynamic momentum coefficient. By casting each update as a proximal step on an approximate two-plane model of the loss, AM adapts to the local geometry of the loss landscape with limited additional hyperparameter overhead. AM integrates seamlessly with standard optimizers such as SGD and AdamW, incurs negligible overhead, and requires no changes to existing training pipelines. Empirically, AM delivers consistent gains across a spectrum of tasks, from convex benchmarks and vision models to large-scale LLM pretraining. In particular, AM-AdamW stabilizes the fragile early phase of LLM training and shows promise in eliminating the need for hand-tuned warm-up schedules, demonstrating robustness to large learning rates and simplifying large-scale workflows. Future work includes: (i) investigating memory in first-order methods and the impact of momentum restarts; (ii) advancing tuning-free momentum-based optimizers with dynamic and per-parameter momentum coefficients; and (iii) developing warm-up-free strategies for stable training.

\newpage
\clearpage
\section*{Acknowledgements}
The research presented in this paper was partially sponsored by the NSF CAREER Award \#2041872.
\printbibliography

\newpage
\clearpage
%%%%%%%%%%%%%%%%%%%%%%%%%%%%%%%%%%%%%%%%%%%%%%%%%%%%%%%%%%%%
\section*{Checklist}

\begin{enumerate}

  \item For all models and algorithms presented, check if you include:
  \begin{enumerate}
    \item A clear description of the mathematical setting, assumptions, algorithm, and/or model. Yes
    \item An analysis of the properties and complexity (time, space, sample size) of any algorithm. Yes
    \item (Optional) Anonymized source code, with specification of all dependencies, including external libraries. Yes
  \end{enumerate}

  \item For any theoretical claim, check if you include:
  \begin{enumerate}
    \item Statements of the full set of assumptions of all theoretical results. Yes
    \item Complete proofs of all theoretical results. Yes
    \item Clear explanations of any assumptions. Yes  
  \end{enumerate}

  \item For all figures and tables that present empirical results, check if you include:
  \begin{enumerate}
    \item The code, data, and instructions needed to reproduce the main experimental results (either in the supplemental material or as a URL). Yes
    \item All the training details (e.g., data splits, hyperparameters, how they were chosen). Yes
    \item A clear definition of the specific measure or statistics and error bars (e.g., with respect to the random seed after running experiments multiple times). Yes
    \item A description of the computing infrastructure used. (e.g., type of GPUs, internal cluster, or cloud provider). Yes
  \end{enumerate}

  \item If you are using existing assets (e.g., code, data, models) or curating/releasing new assets, check if you include:
  \begin{enumerate}
    \item Citations of the creator If your work uses existing assets. Yes
    \item The license information of the assets, if applicable. Not Applicable
    \item New assets either in the supplemental material or as a URL, if applicable. Yes
    \item Information about consent from data providers/curators. Not Applicable
    \item Discussion of sensible content if applicable, e.g., personally identifiable information or offensive content. Not Applicable
  \end{enumerate}

  \item If you used crowdsourcing or conducted research with human subjects, check if you include:
  \begin{enumerate}
    \item The full text of instructions given to participants and screenshots. Not Applicable
    \item Descriptions of potential participant risks, with links to Institutional Review Board (IRB) approvals if applicable. Not Applicable
    \item The estimated hourly wage paid to participants and the total amount spent on participant compensation. Not Applicable
  \end{enumerate}

\end{enumerate}

\clearpage

\appendix
\thispagestyle{empty}

% Supplementary material: To improve readability, you must use a single-column format for the supplementary material.
\onecolumn
\aistatstitle{Supplement}

\section{DERIVATION OF ADAPTIVE MEMORY MOMENTUM}
\label{sec:derivation}
We use the cutting-plane model
\begin{equation}
f_t^m(x)
=
\max\left\{
f(x_t) + g_t^\top (x-x_t),\;
\hat f(x_t) + d_t^\top (x-x_t)
\right\},
\qquad g_t := \nabla f(x_t).
\end{equation}

Let \(P_t\) be a diagonal preconditioner, and define
\begin{equation}
\tilde P_t := (I + \lambda P_t)P_t .
\end{equation}
The update is
\begin{equation}
x_{t+1}
\in
\arg\min_x
\left\{
f_t^m(x)
+
\frac{1}{2\eta}\|x-x_t\|_{\tilde P_t}^2
+
\lambda \langle d_t, x-x_t\rangle_{P_t}
+
\frac{\mu}{2}\|x\|_{\tilde P_t}^2
\right\}.
\end{equation}

Define \(\tilde x := x-x_t\). Then the problem becomes
\begin{equation}
\min_{\tilde x,\zeta}
\left\{
\zeta
+
\frac{1}{2\eta}\|\tilde x\|_{\tilde P_t}^2
+
\lambda \langle d_t,\tilde x\rangle_{P_t}
+
\frac{\mu}{2}\|x_t+\tilde x\|_{\tilde P_t}^2
\right\}
\end{equation}
subject to
\begin{equation}
\zeta \ge f(x_t)+g_t^\top \tilde x,
\qquad
\zeta \ge \hat f(x_t)+d_t^\top \tilde x.
\end{equation}

Introducing dual variables \(a_1,a_2\ge 0\), the Lagrangian is
\begin{align}
\mathcal L(\tilde x,\zeta,a_1,a_2)
&=
\zeta
+
\frac{1}{2\eta}\|\tilde x\|_{\tilde P_t}^2
+
\lambda \langle d_t,\tilde x\rangle_{P_t}
+
\frac{\mu}{2}\|x_t+\tilde x\|_{\tilde P_t}^2
\nonumber\\
&\quad
-a_1\bigl(\zeta-f(x_t)-g_t^\top \tilde x\bigr)
-a_2\bigl(\zeta-\hat f(x_t)-d_t^\top \tilde x\bigr).
\end{align}

The KKT conditions are
\begin{align}
\frac{\partial \mathcal L}{\partial \tilde x}=0
&\iff
\frac{1}{\eta}\tilde P_t \tilde x
+
\lambda P_t d_t
+
\mu \tilde P_t(\tilde x+x_t)
+
a_1 g_t
+
a_2 d_t
=0,
\label{eq:kkt-x}
\\
\frac{\partial \mathcal L}{\partial \zeta}=0
&\iff
a_1+a_2=1.
\label{eq:kkt-zeta}
\end{align}

Solving \eqref{eq:kkt-x} for \(\tilde x\) yields
\begin{equation}
\tilde x
=
-\frac{\eta}{1+\mu\eta}\,
\tilde P_t^{-1}
\Bigl(
(\lambda P_t + a_2 I)d_t
+
\mu \tilde P_t x_t
+
a_1 g_t
\Bigr).
\label{eq:tilde-x}
\end{equation}

Therefore
\begin{equation}
x_{t+1}
=
\frac{1}{1+\mu\eta}
\left(
x_t
-
\eta \tilde P_t^{-1}
\bigl(
a_1 g_t + a_2 d_t + \lambda P_t d_t
\bigr)
\right).
\label{eq:x-update-a}
\end{equation}

Substituting \eqref{eq:tilde-x} back into the objective gives the dual problem
\begin{align}
\max_{\substack{a_1,a_2\ge 0\\ a_1+a_2=1}}
\Biggl\{
-\frac{\eta}{2(1+\mu\eta)}
\left\|
(\lambda P_t+a_2 I)d_t
+
\mu \tilde P_t x_t
+
a_1 g_t
\right\|_{\tilde P_t^{-1}}^2
+
\frac{\mu}{2}\|x_t\|_{\tilde P_t}^2
+
a_1 f(x_t)
+
a_2 \hat f(x_t)
\Biggr\}.
\end{align}

Setting \(a_2=\beta\) and \(a_1=1-\beta\), we obtain
\begin{align}
\max_{0\le \beta \le 1}
\Biggl\{
-\frac{\eta}{2(1+\mu\eta)}
\left\|
\mu \tilde P_t x_t
+
g_t
+
\lambda P_t d_t
+
\beta(d_t-g_t)
\right\|_{\tilde P_t^{-1}}^2
+
\frac{\mu}{2}\|x_t\|_{\tilde P_t}^2
+
(1-\beta)f(x_t)
+
\beta \hat f(x_t)
\Biggr\}.
\end{align}

This is a concave quadratic program. Differentiating with respect to \(\beta\), setting the derivative to zero, and projecting onto \([0,1]\) yields
\begin{equation}
\beta_t^*
=
\operatorname{Clip}_{[0,1]}
\left(
\frac{
\frac{(1+\mu\eta)\bigl(\hat f(x_t)-f(x_t)\bigr)}{\eta}
-
\left\langle d_t-g_t,\; g_t+\lambda P_t d_t \right\rangle_{\tilde P_t^{-1}}
-
\mu \langle x_t,\; d_t-g_t\rangle
}{
\|d_t-g_t\|_{\tilde P_t^{-1}}^2
}
\right).
\label{eq:opt-beta-corrected}
\end{equation}

\begin{itemize}
    \item Setting $\mu=0$ and $P_t=I$ produces Momentum Gradient Descent
    \item Setting $\mu=0$ and $P_t$ equal to the Adam Preconditioner yields Adam.
    \item Setting $\mu>0$ and $P_t$ equal to the Adam Preconditioner yields AdamW.
\end{itemize}

In addition we can interpret Equation~\ref{eq:x-update-a} in two different ways, which are equivalent over one step but differ over what we store in memory.

$$(I):
\begin{cases}
    d_{t+1} = (\lambda P_t + I)^{-1}\left( \left(1-\beta_t\right)g_{t} +    \left(\lambda P_t + \beta_tI\right)d_{t} \right)\\
    x_{t+1} =\frac{1}{1+\mu\eta}\bigg(x_t -\eta P_t^{-1}d_{t+1}\bigg)
\end{cases}
$$

or 

$$
(II):
\begin{cases}
    d_{t+1} =  \left(1-\beta_t\right)g_{t}  + \beta_td_{t} \\
    x_{t+1} =\frac{1}{1+\mu\eta}\bigg(x_t -\eta P_t^{-1}(\lambda P_t + I)^{-1}\left( \left(1-\beta_t\right)g_{t} +    \left(\lambda P_t + \beta_tI\right)d_{t} \right)\bigg)
\end{cases}
$$

Of the two options, the first was found to perform slightly better. Additionally, it aligns with existing momentum-based methods. In contrast, the second option requires extra memory to store both \( d_{t+1} \), which represents the momentum for the next iteration, and \( d_t \), which is needed for the parameter update.

\section{DISCUSSION ON THE OVERESTIMATING MOMENTUM PLANE}
\label{sec:overestimating}

In this section we give a simple motivating example showing that, in a
non-overshooting momentum regime, a positive \emph{overestimation term} can be
aligned with the one-step optimal choice of momentum. We focus on heavy-ball
(HB) momentum on a diagonal strongly convex quadratic, where the dynamics
decouple across coordinates and can be analyzed exactly.

\subsection{Non-overshooting heavy-ball dynamics}

We consider the quadratic
\[
f(x)=\tfrac12 x^\top A x,
\qquad
A=\operatorname{diag}(a_1,\dots,a_d),
\qquad
0<a_d\le \cdots \le a_1=L,
\]
whose minimizer is \(x^*=0\). The heavy-ball iteration is
\[
d_{t+1}=\beta d_t+(1-\beta)Ax_t,
\qquad
x_{t+1}=x_t-\eta d_{t+1},
\]
with \(\beta\in[0,1)\), \(\eta>0\), and initialization
\[
d_0 = g_0 = Ax_0.
\]

\begin{definition}[Coordinate-wise monotonic decrease]
We say that the iterates satisfy the \emph{coordinate-wise monotonic decrease
condition} if for every \(t\ge 0\) and every coordinate \(i\),
\[
|x_{t+1}^i-x^{*i}| \le |x_t^i-x^{*i}|
\quad\text{and}\quad
\operatorname{sign}(x_{t+1}^i-x^{*i})
=
\operatorname{sign}(x_t^i-x^{*i}).
\]
Since \(x^*=0\) here, this is equivalent to
\[
|x_{t+1}^i| \le |x_t^i|
\quad\text{and}\quad
\operatorname{sign}(x_{t+1}^i)=\operatorname{sign}(x_t^i).
\]
\end{definition}

For the diagonal quadratic above, plain gradient descent corresponds to
\(\beta=0\) and satisfies this condition whenever \(\eta\le 1/L\). The next
proposition shows that heavy-ball also satisfies it in the overdamped regime.

\begin{proposition}[No-overshoot regime of HB on a diagonal quadratic]
\label{prop:hb_no_overshoot}
Assume
\[
0<\beta<1,
\qquad
\eta L < \frac{1-\sqrt{\beta}}{1+\sqrt{\beta}}.
\]
Then for every coordinate \(i\) with \(x_0^i\neq 0\) and every \(t\ge 0\),
\[
\operatorname{sign}(x_t^i)=\operatorname{sign}(x_0^i),
\qquad
\operatorname{sign}(g_t^i)=\operatorname{sign}(x_0^i),
\qquad
\operatorname{sign}(d_t^i)=\operatorname{sign}(x_0^i),
\]
and
\[
|x_{t+1}^i| < |x_t^i|.
\]
If \(x_0^i=0\), then \(x_t^i=d_t^i=g_t^i=0\) for all \(t\).
\end{proposition}

\begin{proof}
Fix a coordinate \(i\) and write \(a:=a_i>0\). If \(x_0^i=0\), then
\(d_0^i=g_0^i=0\), and the scalar recursion stays identically zero. We
therefore assume \(x_0^i\neq 0\). By symmetry it is enough to consider the case
\(x_0^i>0\); the case \(x_0^i<0\) follows by multiplying the scalar dynamics by
\(-1\).

The scalar HB recursion is
\[
d_{t+1}=\beta d_t+(1-\beta)ax_t,
\qquad
x_{t+1}=x_t-\eta d_{t+1}.
\]
Define \(z_t=(d_t,x_t)^\top\). Then
\[
z_{t+1}=Mz_t,
\qquad
M=
\begin{pmatrix}
\beta & (1-\beta)a\\
-\eta\beta & 1-\eta(1-\beta)a
\end{pmatrix}.
\]
The characteristic polynomial of \(M\) is
\[
r^2-\tau r+\beta=0,
\qquad
\tau = 1+\beta-\eta(1-\beta)a.
\]
Since \(a\le L\) and
\[
\eta L < \frac{1-\sqrt{\beta}}{1+\sqrt{\beta}},
\]
we also have
\[
\eta a < \frac{1-\sqrt{\beta}}{1+\sqrt{\beta}}.
\]
This is exactly the overdamped condition, and it implies that the two roots
\(r_1,r_2\) satisfy
\[
0<r_2<r_1<1.
\]

For each root \(r_j\), an eigenvector is
\[
u_j=
\begin{pmatrix}
-\theta_j\\
1
\end{pmatrix},
\qquad
\theta_j=\frac{(1-\beta)a}{\beta-r_j},
\qquad j=1,2.
\]
Hence
\[
x_t=c_1r_1^t+c_2r_2^t.
\]
Using the initialization
\[
(d_0,x_0)^\top = (ax_0,x_0)^\top = c_1u_1+c_2u_2,
\]
we obtain
\[
c_1=-\frac{x_0(a+\theta_2)}{\theta_1-\theta_2},
\qquad
c_2=\frac{x_0(a+\theta_1)}{\theta_1-\theta_2}.
\]

Because \(0<r_2<r_1<1\) and \(\beta=r_1r_2\), we have
\[
\beta-r_1<\beta-r_2<0,
\]
so
\[
\theta_2<\theta_1<0.
\]
Moreover,
\[
a+\theta_j
=
a+\frac{(1-\beta)a}{\beta-r_j}
=
\frac{(1-r_j)a}{\beta-r_j}
<0,
\qquad j=1,2.
\]
Since \(x_0>0\), this shows that \(c_1>0\) and \(c_2<0\). Also,
\[
\frac{|c_2|}{c_1}
=
\frac{-(a+\theta_1)}{-(a+\theta_2)}
<1,
\]
because \(\theta_1>\theta_2\) implies \(a+\theta_1>a+\theta_2\), and both
quantities are negative. Therefore
\[
c_1>|c_2|.
\]
It follows that for every \(t\ge 0\),
\[
x_t
=
r_1^t\!\left(c_1+c_2(r_2/r_1)^t\right)
=
r_1^t\!\left(c_1-|c_2|(r_2/r_1)^t\right)
\ge
r_1^t(c_1-|c_2|)
=
r_1^t x_0
>0.
\]
Thus \(x_t>0\) for all \(t\), so the sign of \(x_t\) never changes.

Since \(g_t=ax_t\), we also have \(g_t>0\) for all \(t\). Finally, because
\(d_0=ax_0>0\) and
\[
d_{t+1}=\beta d_t+(1-\beta)g_t
\]
is a convex combination of two positive quantities, it follows by induction that
\(d_t>0\) for all \(t\). Therefore
\[
x_{t+1}=x_t-\eta d_{t+1}<x_t.
\]
Combined with \(x_{t+1}>0\), this yields
\[
0<x_{t+1}<x_t,
\]
which is exactly the claimed coordinate-wise monotonic decrease condition. The
case \(x_0^i<0\) follows identically by symmetry.
\end{proof}

\subsection{Optimal one-step momentum}

A standard way to motivate adaptive step sizes is to choose the learning rate
\(\eta_t\) by minimizing the one-step distance to the optimum,
\[
J(\eta_t)=\|x_{t+1}(\eta_t)-x^*\|^2,
\]
which leads to the classical Polyak step size after eliminating the unknown
inner product by convexity. One can do the same with the momentum parameter.

\begin{proposition}[Optimal one-step momentum parameter]
\label{prop:one_step_beta}
Fix an iterate \(x_t\), a direction estimate \(d_t\), and the gradient
\(g_t=\nabla f(x_t)\). For \(\beta_t\in[0,1]\), define
\[
d_{t+1}(\beta_t)=\beta_t d_t+(1-\beta_t)g_t,
\qquad
x_{t+1}(\beta_t)=x_t-\eta d_{t+1}(\beta_t).
\]
Then the choice of \(\beta_t\) minimizing the one-step distance to the optimum,
\[
\beta_t^*
\in
\arg\min_{\beta\in[0,1]}
\|x_{t+1}(\beta)-x^*\|^2,
\]
is
\[
\beta_t^*
=
\operatorname{Clip}_{[0,1]}
\left(
\frac{
\frac{(d_t-g_t)^\top(x_t-x^*)}{\eta}
-d_t^\top g_t+\|g_t\|^2
}{
\|d_t-g_t\|^2
}
\right).
\]
\end{proposition}

\begin{proof}
Define
\[
J(\beta_t)
=
\|x_{t+1}(\beta_t)-x^*\|^2
=
\left\|
(x_t-x^*)-\eta\bigl[\beta_t d_t+(1-\beta_t)g_t\bigr]
\right\|^2.
\]
Set
\[
u=(x_t-x^*)-\eta g_t,
\qquad
v=d_t-g_t.
\]
Then
\[
x_{t+1}(\beta_t)-x^*=u-\eta\beta_t v,
\]
and therefore
\[
J(\beta_t)=\|u\|^2-2\eta\beta_t\,u^\top v+\eta^2\beta_t^2\|v\|^2.
\]
Differentiating and setting the derivative to zero gives
\[
\beta_t=\frac{u^\top v}{\eta\|v\|^2}.
\]
Substituting back \(u\) and \(v\) yields
\[
\beta_t
=
\frac{\bigl((x_t-x^*)-\eta g_t\bigr)^\top(d_t-g_t)}
{\eta\|d_t-g_t\|^2}
=
\frac{
\frac{(d_t-g_t)^\top(x_t-x^*)}{\eta}
-d_t^\top g_t+\|g_t\|^2
}{
\|d_t-g_t\|^2
}.
\]
Projecting onto \([0,1]\) gives the stated formula.
\end{proof}

The only term in Proposition~\ref{prop:one_step_beta} that depends explicitly on
the unknown optimum is
\[
(d_t-g_t)^\top(x_t-x^*).
\]
We now show that in the non-overshooting quadratic regime above, this term is
coordinate-wise nonnegative.

\begin{lemma}
\label{lem:positive_alignment}
Under the assumptions of Proposition~\ref{prop:hb_no_overshoot}, for every
\(t\ge 0\) and every coordinate \(i\),
\[
(d_t^i-g_t^i)x_t^i \ge 0.
\]
Consequently,
\[
(d_t-g_t)^\top(x_t-x^*)=(d_t-g_t)^\top x_t \ge 0.
\]
\end{lemma}

\begin{proof}
We argue by induction on \(t\).

For \(t=0\), since \(d_0^i=g_0^i=a_i x_0^i\),
\[
(d_0^i-g_0^i)x_0^i = 0.
\]

Now assume \((d_t^i-g_t^i)x_t^i\ge 0\). By
Proposition~\ref{prop:hb_no_overshoot}, \(x_t^i\) and \(x_{t+1}^i\) have the
same sign, and \(g_t^i=a_i x_t^i\) has that sign as well. Also,
\[
d_{t+1}^i=\beta d_t^i+(1-\beta)g_t^i,
\qquad
g_{t+1}^i=a_i x_{t+1}^i.
\]
A direct calculation gives
\[
d_{t+1}^i-g_{t+1}^i
=
\beta(1+\eta a_i)(d_t^i-g_t^i)+\eta a_i g_t^i.
\]
By the induction hypothesis, \(d_t^i-g_t^i\) has the same sign as \(x_t^i\),
and \(g_t^i\) also has the same sign as \(x_t^i\). Since all coefficients are
nonnegative, the right-hand side has the same sign as \(x_t^i\), hence also the
same sign as \(x_{t+1}^i\). Therefore
\[
(d_{t+1}^i-g_{t+1}^i)x_{t+1}^i \ge 0.
\]
This proves the coordinate-wise claim. Summing over \(i\) and using \(x^*=0\)
gives
\[
(d_t-g_t)^\top(x_t-x^*) = \sum_{i=1}^d (d_t^i-g_t^i)x_t^i \ge 0.
\]
\end{proof}

Lemma~\ref{lem:positive_alignment} shows that, in this non-overshooting
quadratic regime, the one-step optimal momentum parameter contains a
nonnegative contribution from the term
\[
\frac{(d_t-g_t)^\top(x_t-x^*)}{\eta}.
\]
This provides a concrete setting in which replacing the unknown optimum-dependent
quantity by a positive surrogate, equivalently, using an \emph{overestimating}
model term, is directionally consistent with improving one-step progress.

\section{PROOFS OF SECTION \ref{theory}}

We provide convergence guarantees for our proposed method in the standard finite-sum setting:
\[
\min_{x \in \mathbb{R}^d} f(x)
\;=\;
\frac{1}{n}\sum_{i=1}^n f_{S_i}(x),
\]
where each \(f_{S_i}\) denotes the loss associated with a mini-batch \(S_i\), and the goal is to minimize the average loss \(f\). We proceed under the following assumptions.

\begin{assumption}[Smoothness]
The function \(f\) is \(L\)-smooth if, for all \(x,y \in \mathbb{R}^d\),
\begin{equation}
f(y)
\le
f(x) + \nabla f(x)^\top (y-x) + \frac{L}{2}\|y-x\|^2.
\end{equation}
\end{assumption}

\begin{assumption}[Bounded stochastic gradient norm]
There exists \(G>0\) such that, for all \(x\),
\begin{equation}
\mathbb{E}_{S}\!\left[\|\nabla f_{S}(x)\|^2\right]
\le
G^2.
\label{ass:bounded_stoch_grad}
\end{equation}
\end{assumption}

\paragraph{Algorithm.}
We prove convergence for a simplified version of AM-SGD-M with \(\lambda=0\).
Let \(\eta>0\) be the step size, let \(0 \le \beta_t \le \beta_{\max}<1\), and initialize \(d_0=0\).
At iteration \(t\), let
\[
g_t := \nabla f_{S_t}(x_t),
\]
where \(S_t\) denotes the sampled mini-batch. The updates are
\begin{equation}
\label{eq:mom_for_proof}
\begin{aligned}
\beta_t
&=
\operatorname{Clip}_{[0,\beta_{\max}]}
\left(
\frac{\|g_t\|^2}{\|d_t-g_t\|^2}
\right),\\
d_{t+1}
&=
\beta_t d_t + (1-\beta_t) g_t,\\
x_{t+1}
&=
x_t - \eta d_{t+1}.
\end{aligned}
\end{equation}
When \(d_t=g_t\), the denominator vanishes; in that case we define \(\beta_t:=0\) by convention, noting that then \(d_{t+1}=g_t\) regardless of the value of \(\beta_t\).

We denote by \(x^*\) an optimal point, i.e. \(x^* \in \arg\min_x f(x)\).

\begin{lemma}
\label{lem:velocity_bound}
Consider the iterates generated by \eqref{eq:mom_for_proof}. Assume that
\[
\mathbb{E}\bigl[\|g_t\|^2\bigr] \le G^2.
\]
Then, for every \(t\),
\[
\mathbb{E}\bigl[\|d_{t+1}\|^2\bigr] \le 4G^2.
\]
\end{lemma}

\begin{proof}
We rewrite the momentum update as
\[
d_{t+1}
=
g_t + \beta_t(d_t-g_t).
\]
Using Young's inequality, for any \(\rho>0\),
\[
\|a+b\|^2
\le
(1+\rho)\|a\|^2 + \left(1+\frac{1}{\rho}\right)\|b\|^2.
\]
Applying this with
\[
a=g_t,
\qquad
b=\beta_t(d_t-g_t),
\]
gives
\[
\|d_{t+1}\|^2
\le
(1+\rho)\|g_t\|^2
+
\left(1+\frac{1}{\rho}\right)\beta_t^2\|d_t-g_t\|^2.
\]
By the definition of \(\beta_t\),
\[
\beta_t
\le
\frac{\|g_t\|^2}{\|d_t-g_t\|^2}
\]
whenever \(d_t\neq g_t\), and trivially also when \(d_t=g_t\) under the convention above. Hence
\[
\beta_t \|d_t-g_t\|^2 \le \|g_t\|^2.
\]
Since \(0\le \beta_t \le \beta_{\max}<1\), we also have \(\beta_t^2 \le \beta_t\), so
\[
\beta_t^2\|d_t-g_t\|^2
\le
\beta_t\|d_t-g_t\|^2
\le
\|g_t\|^2.
\]
Substituting this into the previous bound yields
\[
\|d_{t+1}\|^2
\le
\left(2+\rho+\frac{1}{\rho}\right)\|g_t\|^2.
\]
Choosing \(\rho=1\), we obtain
\[
\|d_{t+1}\|^2 \le 4\|g_t\|^2.
\]
Taking expectation and using \(\mathbb{E}[\|g_t\|^2]\le G^2\) gives
\[
\mathbb{E}\bigl[\|d_{t+1}\|^2\bigr]
\le
4G^2.
\]
\end{proof}

The following technical lemmas provide upper bounds on the cross terms that arise when expanding the momentum recursion. Both results rely on the definition of \(\beta_t\) to relate the correction term \(d_t-g_t\) to the stochastic gradient itself, and on Young's inequality to separate mixed inner products.
\begin{lemma}
\label{lem:young_momentum_bound_1}
Consider the iterates generated by \eqref{eq:mom_for_proof}. For any \(a>0\),
\[
-2\eta\beta_t \langle x_t-x^*,\, d_t-g_t\rangle
\;\le\;
\beta_{\max} a \|x_t-x^*\|^2
+
\frac{\eta^2}{a}\|g_t\|^2,
\]
where \(g_t := \nabla f_{S_t}(x_t)\).
\end{lemma}

\begin{proof}
By Cauchy--Schwarz,
\[
-2\eta\beta_t \langle x_t-x^*,\, d_t-g_t\rangle
\;\le\;
2\eta\beta_t \|x_t-x^*\|\,\|d_t-g_t\|.
\]
Applying Young's inequality \(2uv \le a u^2 + \frac{1}{a}v^2\) with
\[
u=\sqrt{\beta_t}\,\|x_t-x^*\|,
\qquad
v=\eta\sqrt{\beta_t}\,\|d_t-g_t\|,
\]
gives
\[
2\eta\beta_t \|x_t-x^*\|\,\|d_t-g_t\|
\;\le\;
\beta_t a \|x_t-x^*\|^2
+
\frac{\eta^2\beta_t}{a}\|d_t-g_t\|^2.
\]
Since \(\beta_t \le \beta_{\max}\), the first term satisfies
\[
\beta_t a \|x_t-x^*\|^2
\;\le\;
\beta_{\max} a \|x_t-x^*\|^2.
\]
Moreover, by the definition of \(\beta_t\),
\[
\beta_t \|d_t-g_t\|^2 \le \|g_t\|^2,
\]
with the inequality being trivial when \(d_t=g_t\). Therefore,
\[
\frac{\eta^2\beta_t}{a}\|d_t-g_t\|^2
\;\le\;
\frac{\eta^2}{a}\|g_t\|^2.
\]
Combining the two bounds yields
\[
-2\eta\beta_t \langle x_t-x^*,\, d_t-g_t\rangle
\;\le\;
\beta_{\max} a \|x_t-x^*\|^2
+
\frac{\eta^2}{a}\|g_t\|^2.
\]
\end{proof}

\begin{lemma}
\label{lem:young_momentum_bound_2}
Consider the iterates generated by \eqref{eq:mom_for_proof}. For any \(a_t>0\),
\[
-\eta\beta_t \langle \nabla f(x_t),\, d_t-g_t\rangle
\;\le\;
\frac{\beta_{\max}}{2a_t}\|\nabla f(x_t)\|^2
+
\frac{a_t\eta^2}{2}\|g_t\|^2,
\]
where \(g_t := \nabla f_{S_t}(x_t)\).
\end{lemma}

\begin{proof}
By Cauchy--Schwarz,
\[
-\eta\beta_t \langle \nabla f(x_t),\, d_t-g_t\rangle
\;\le\;
\eta\beta_t \|\nabla f(x_t)\|\,\|d_t-g_t\|.
\]
Applying Young's inequality \(uv \le \frac{u^2}{2a_t} + \frac{a_t}{2}v^2\) with
\[
u=\sqrt{\beta_t}\,\|\nabla f(x_t)\|,
\qquad
v=\eta\sqrt{\beta_t}\,\|d_t-g_t\|,
\]
gives
\[
\eta\beta_t \|\nabla f(x_t)\|\,\|d_t-g_t\|
\;\le\;
\frac{\beta_t}{2a_t}\|\nabla f(x_t)\|^2
+
\frac{a_t\eta^2\beta_t}{2}\|d_t-g_t\|^2.
\]
Using \(\beta_t \le \beta_{\max}\) for the first term and
\[
\beta_t\|d_t-g_t\|^2 \le \|g_t\|^2
\]
for the second term, we obtain
\[
-\eta\beta_t \langle \nabla f(x_t),\, d_t-g_t\rangle
\;\le\;
\frac{\beta_{\max}}{2a_t}\|\nabla f(x_t)\|^2
+
\frac{a_t\eta^2}{2}\|g_t\|^2.
\]
\end{proof}

\begin{theorem}[Convex convergence]
\label{thm:convex_app}
Assume that each stochastic component \(f_{S}\) is convex, that the stochastic
gradient is unbiased,
\[
\mathbb{E}[\,\nabla f_{S_t}(x_t)\mid x_t\,]=\nabla f(x_t),
\]
and that Assumption~\ref{ass:bounded_stoch_grad} holds. Let
\(x^* \in \arg\min_x f(x)\). Consider the iterates generated by
\eqref{eq:mom_for_proof} with
\[
\eta=\frac{1}{\sqrt T},
\qquad
\beta_{\max}=\frac{1}{T}.
\]
Define the weights
\[
a_t:=\left(1+\frac{1}{T}\right)^{-(t+1)},
\qquad t=0,\dots,T-1,
\]
and the weighted average iterate
\[
x_T^{(a)}
:=
\frac{\sum_{t=0}^{T-1} a_t x_t}{\sum_{t=0}^{T-1} a_t}.
\]
Then

\[
\mathbb{E}\bigl[f(x_T^{(a)})-f(x^*)\bigr]
\le
\frac{1}{2\sqrt T}
\left(
2\|x_0-x^*\|^2
+
5G^2
\right).
\]
\end{theorem}

\begin{proof}
Let
\[
g_t:=\nabla f_{S_t}(x_t).
\]
From the update \(x_{t+1}=x_t-\eta d_{t+1}\), we have
\begin{align*}
\|x_{t+1}-x^*\|^2
&=
\|x_t-x^*\|^2
-2\eta\langle x_t-x^*,\,d_{t+1}\rangle
+\eta^2\|d_{t+1}\|^2 \\
&=
\|x_t-x^*\|^2
-2\eta\langle x_t-x^*,\,g_t\rangle
-2\eta\beta_t\langle x_t-x^*,\,d_t-g_t\rangle
+\eta^2\|d_{t+1}\|^2.
\end{align*}
By convexity of \(f_{S_t}\),
\[
\langle x_t-x^*,\,g_t\rangle
\ge
f_{S_t}(x_t)-f_{S_t}(x^*).
\]
Applying Lemma~\ref{lem:young_momentum_bound_1} with \(a=1\),
\[
-2\eta\beta_t\langle x_t-x^*,\,d_t-g_t\rangle
\le
\beta_{\max}\|x_t-x^*\|^2+\eta^2\|g_t\|^2.
\]
Therefore,
\[
\|x_{t+1}-x^*\|^2
\le
(1+\beta_{\max})\|x_t-x^*\|^2
-2\eta\bigl(f_{S_t}(x_t)-f_{S_t}(x^*)\bigr)
+\eta^2\|g_t\|^2+\eta^2\|d_{t+1}\|^2.
\]
Taking expectation, using unbiasedness, Assumption~\ref{ass:bounded_stoch_grad},
and Lemma~\ref{lem:velocity_bound}, we obtain
\[
\mathbb{E}\bigl[\|x_{t+1}-x^*\|^2\bigr]
\le
(1+\beta_{\max})\mathbb{E}\bigl[\|x_t-x^*\|^2\bigr]
-2\eta\,\mathbb{E}\bigl[f(x_t)-f(x^*)\bigr]
+5\eta^2G^2.
\]

Now define the weights recursively by
\[
a_{t-1}=(1+\beta_{\max})a_t,
\qquad
a_{-1}=1.
\]
Since \(\beta_{\max}=1/T\), this yields
\[
a_t=\left(1+\frac{1}{T}\right)^{-(t+1)}.
\]
Multiplying the previous inequality by \(a_t\) gives
\[
a_t\,\mathbb{E}\bigl[\|x_{t+1}-x^*\|^2\bigr]
\le
a_{t-1}\,\mathbb{E}\bigl[\|x_t-x^*\|^2\bigr]
-2\eta a_t\,\mathbb{E}\bigl[f(x_t)-f(x^*)\bigr]
+5\eta^2G^2a_t.
\]
Rearranging and summing from \(t=0\) to \(T-1\),
\[
2\eta\sum_{t=0}^{T-1} a_t\,\mathbb{E}\bigl[f(x_t)-f(x^*)\bigr]
\le
\|x_0-x^*\|^2
-
a_{T-1}\mathbb{E}\bigl[\|x_T-x^*\|^2\bigr]
+
5\eta^2G^2\sum_{t=0}^{T-1} a_t.
\]
Dropping the nonnegative term
\(a_{T-1}\mathbb{E}[\|x_T-x^*\|^2]\), we get
\[
2\eta\sum_{t=0}^{T-1} a_t\,\mathbb{E}\bigl[f(x_t)-f(x^*)\bigr]
\le
\|x_0-x^*\|^2
+
5\eta^2G^2\sum_{t=0}^{T-1} a_t.
\]

Since \(f\) is convex,
\[
f(x_T^{(a)})
\le
\sum_{t=0}^{T-1}\frac{a_t}{\sum_{s=0}^{T-1}a_s}f(x_t).
\]
Taking expectation and subtracting \(f(x^*)\),
\[
\mathbb{E}\bigl[f(x_T^{(a)})-f(x^*)\bigr]
\le
\frac{\sum_{t=0}^{T-1} a_t\,\mathbb{E}[f(x_t)-f(x^*)]}
{\sum_{t=0}^{T-1} a_t}.
\]
Combining with the previous bound yields
\[
\mathbb{E}\bigl[f(x_T^{(a)})-f(x^*)\bigr]
\le
\frac{\|x_0-x^*\|^2}{2\eta\sum_{t=0}^{T-1} a_t}
+
\frac{5}{2}\eta G^2.
\]

It remains to lower bound \(\sum_{t=0}^{T-1} a_t\). Since
\[
a_t=\left(1+\frac{1}{T}\right)^{-(t+1)},
\]
we have
\[
\sum_{t=0}^{T-1} a_t
=
\frac{1-(1+\frac{1}{T})^{-T}}{\frac{1}{T}\cdot(1+\frac{1}{T})^{-1}}\cdot\frac{1}{T}
=
T\left(1-\left(1+\frac{1}{T}\right)^{-T}\right).
\]
Since \((1+1/T)^{T}\) is increasing in \(T\) and equals \(2\) at \(T=1\),
we have \((1+1/T)^{-T}\le 1/2\) for all \(T\ge 1\), and therefore
\[
\sum_{t=0}^{T-1} a_t
\;\ge\;
\frac{T}{2}.
\]
Using \(\eta=1/\sqrt{T}\), we conclude
\[
\mathbb{E}\bigl[f(x_T^{(a)})-f(x^*)\bigr]
\le
\frac{\|x_0-x^*\|^2}{2\eta\cdot T/2}
+
\frac{5}{2}\eta G^2
=
\frac{\|x_0-x^*\|^2}{\sqrt{T}}
+
\frac{5G^2}{2\sqrt{T}}
=
\frac{1}{2\sqrt T}
\left(
2\|x_0-x^*\|^2
+
5G^2
\right).
\]
\end{proof}

\begin{theorem}[Non-convex convergence]
\label{thm:nonconvex-convergence-app}
Assume that \(f\) is \(L\)-smooth, that the stochastic gradient is unbiased,
\[
\mathbb{E}[\,\nabla f_{S_t}(x_t)\mid x_t\,]=\nabla f(x_t),
\]
and that Assumption~\ref{ass:bounded_stoch_grad} holds. Consider the iterates
generated by \eqref{eq:mom_for_proof} with
\[
\eta=\frac{1}{\sqrt T},
\qquad
\beta_{\max}=c\eta,
\qquad
0<c<1.
\]
Then
\[
\min_{0\le t\le T-1}\mathbb{E}\bigl[\|\nabla f(x_t)\|^2\bigr]
\le
\frac{1}{\sqrt T(1-c)}
\left(
\mathbb{E}[f(x_0)-f(x_T)]
+
\frac{1+8L}{4}G^2
\right).
\]
\end{theorem}

\begin{proof}
Let
\[
g_t:=\nabla f_{S_t}(x_t).
\]
By \(L\)-smoothness of \(f\),
\begin{align*}
f(x_{t+1})
&\le
f(x_t)
+\langle \nabla f(x_t),\,x_{t+1}-x_t\rangle
+\frac{L}{2}\|x_{t+1}-x_t\|^2 \\
&=
f(x_t)
-\eta\langle \nabla f(x_t),\,d_{t+1}\rangle
+\frac{L\eta^2}{2}\|d_{t+1}\|^2.
\end{align*}
Using
\[
d_{t+1}=g_t+\beta_t(d_t-g_t),
\]
we obtain
\[
f(x_{t+1})
\le
f(x_t)
-\eta\langle \nabla f(x_t),\,g_t\rangle
-\eta\beta_t\langle \nabla f(x_t),\,d_t-g_t\rangle
+\frac{L\eta^2}{2}\|d_{t+1}\|^2.
\]

Applying Lemma~\ref{lem:young_momentum_bound_2} with \(a_t=\tfrac12\),
\[
-\eta\beta_t\langle \nabla f(x_t),\,d_t-g_t\rangle
\le
\beta_{\max}\|\nabla f(x_t)\|^2
+\frac{\eta^2}{4}\|g_t\|^2.
\]
Hence,
\[
f(x_{t+1})
\le
f(x_t)
-\eta\langle \nabla f(x_t),\,g_t\rangle
+\beta_{\max}\|\nabla f(x_t)\|^2
+\frac{\eta^2}{4}\|g_t\|^2
+\frac{L\eta^2}{2}\|d_{t+1}\|^2.
\]

Taking expectation and using unbiasedness,
\[
\mathbb{E}\bigl[\langle \nabla f(x_t),\,g_t\rangle\bigr]
=
\mathbb{E}\bigl[\|\nabla f(x_t)\|^2\bigr].
\]
Using Assumption~\ref{ass:bounded_stoch_grad} and
Lemma~\ref{lem:velocity_bound},
\[
\mathbb{E}\bigl[\|g_t\|^2\bigr]\le G^2,
\qquad
\mathbb{E}\bigl[\|d_{t+1}\|^2\bigr]\le 4G^2.
\]
Therefore,
\[
\mathbb{E}[f(x_{t+1})]
\le
\mathbb{E}[f(x_t)]
-\bigl(\eta-\beta_{\max}\bigr)\mathbb{E}\bigl[\|\nabla f(x_t)\|^2\bigr]
+\eta^2G^2\left(\frac14+2L\right).
\]
Equivalently,
\[
\bigl(\eta-\beta_{\max}\bigr)\mathbb{E}\bigl[\|\nabla f(x_t)\|^2\bigr]
\le
\mathbb{E}[f(x_t)-f(x_{t+1})]
+\eta^2G^2\left(\frac14+2L\right).
\]

Summing from \(t=0\) to \(T-1\),
\[
\bigl(\eta-\beta_{\max}\bigr)
\sum_{t=0}^{T-1}\mathbb{E}\bigl[\|\nabla f(x_t)\|^2\bigr]
\le
\mathbb{E}[f(x_0)-f(x_T)]
+
T\eta^2G^2\left(\frac14+2L\right).
\]
Since
\[
\frac14+2L=\frac{1+8L}{4},
\]
we obtain
\[
\bigl(\eta-\beta_{\max}\bigr)
\sum_{t=0}^{T-1}\mathbb{E}\bigl[\|\nabla f(x_t)\|^2\bigr]
\le
\mathbb{E}[f(x_0)-f(x_T)]
+
\frac{1+8L}{4}\,T\eta^2G^2.
\]

Using \(\beta_{\max}=c\eta\), we have
\[
\eta-\beta_{\max}=\eta(1-c).
\]
Therefore,
\[
\sum_{t=0}^{T-1}\mathbb{E}\bigl[\|\nabla f(x_t)\|^2\bigr]
\le
\frac{1}{\eta(1-c)}
\left(
\mathbb{E}[f(x_0)-f(x_T)]
+
\frac{1+8L}{4}\,T\eta^2G^2
\right).
\]
Dividing by \(T\) and using \(\eta=1/\sqrt T\),
\[
\frac{1}{T}\sum_{t=0}^{T-1}\mathbb{E}\bigl[\|\nabla f(x_t)\|^2\bigr]
\le
\frac{1}{\sqrt T(1-c)}
\left(
\mathbb{E}[f(x_0)-f(x_T)]
+
\frac{1+8L}{4}G^2
\right).
\]
Finally,
\[
\min_{0\le t\le T-1}\mathbb{E}\bigl[\|\nabla f(x_t)\|^2\bigr]
\le
\frac{1}{T}\sum_{t=0}^{T-1}\mathbb{E}\bigl[\|\nabla f(x_t)\|^2\bigr],
\]
which proves the claim.
\end{proof}

\subsection{A fixed-cap variant.}
The main text uses a shrinking cap on \(\beta_{\max}\) to obtain convergence to the exact minimizer. If instead the cap is kept fixed, one can still obtain an \(O(1/\sqrt{T})\) convergence rate to a noise-dependent neighborhood under the following finite-gap condition.

\begin{assumption}[Finite optimal mini-batch gap]
\label{ass:finite_gap}
Let \(x^* \in \arg\min_x f(x)\), and define
\[
f_S^* := \inf_x f_S(x).
\]
Assume there exists \(\sigma_f^2 < \infty\) such that
\[
\mathbb{E}_{S}\!\left[f_S(x^*) - f_S^*\right] \le \sigma_f^2.
\]
\end{assumption}

We consider the same simplified AM-SGD-M update as in \eqref{eq:mom_for_proof}, but now with a fixed cap
\[
0 \le \beta_t \le \beta_{\max} < 1.
\]
For convenience, define
\[
\epsilon := 1-\beta_{\max} > 0.
\]

\begin{lemma}
\label{lem:fixedcap_velocity}
Under Assumption~\ref{growth}, for every \(t\),
\[
\mathbb{E}\bigl[\|d_{t+1}\|^2\bigr] \le \frac{G^2}{\epsilon}.
\]
\end{lemma}

\begin{proof}
Using the update \(d_{t+1}=\beta_t d_t+(1-\beta_t)g_t\) and Jensen's inequality,
\[
\|d_{t+1}\|^2
=
\|\beta_t d_t + (1-\beta_t) g_t\|^2
\le
\beta_t \|d_t\|^2 + (1-\beta_t)\|g_t\|^2
\le
\beta_{\max}\|d_t\|^2 + \|g_t\|^2.
\]
Taking expectation and using Assumption~\ref{growth},
\[
\mathbb{E}\bigl[\|d_{t+1}\|^2\bigr]
\le
\beta_{\max}\,\mathbb{E}\bigl[\|d_t\|^2\bigr] + G^2.
\]
Since \(d_0=0\), unrolling the recursion gives
\[
\mathbb{E}\bigl[\|d_{t+1}\|^2\bigr]
\le
G^2\sum_{k=0}^{t}\beta_{\max}^k
\le
\frac{G^2}{1-\beta_{\max}}
=
\frac{G^2}{\epsilon}.
\]
\end{proof}

\begin{lemma}
\label{lem:fixedcap_cross}
Assume each \(f_S\) is convex, Assumption~\ref{growth} holds, and Assumption~\ref{ass:finite_gap} holds. Then, for every \(t\),
\[
\mathbb{E}\bigl[-\beta_t \langle x_t-x^*,\, d_t\rangle\bigr]
\le
\frac{\eta G^2}{\epsilon^2} + \frac{\sigma_f^2}{\epsilon}.
\]
\end{lemma}

\begin{proof}
Set
\[
A_t := -\langle x_t-x^*,\, d_t\rangle.
\]
Using \(x_t=x_{t-1}-\eta d_t\), we write
\begin{align*}
A_t
&=
-\langle x_t-x^*,\,d_t\rangle \\
&=
-\langle x_t-x_{t-1},\,d_t\rangle
-\langle x_{t-1}-x^*,\,d_t\rangle \\
&=
\eta\|d_t\|^2
-\langle x_{t-1}-x^*,\,\beta_{t-1}d_{t-1}+(1-\beta_{t-1})g_{t-1}\rangle \\
&=
\eta\|d_t\|^2
+\beta_{t-1}A_{t-1}
-(1-\beta_{t-1})\langle x_{t-1}-x^*,\,g_{t-1}\rangle.
\end{align*}
By convexity of \(f_{S_{t-1}}\),
\[
\langle x_{t-1}-x^*,\,g_{t-1}\rangle
\ge
f_{S_{t-1}}(x_{t-1})-f_{S_{t-1}}(x^*)
\ge
f_{S_{t-1}}^*-f_{S_{t-1}}(x^*).
\]
Hence
\[
-(1-\beta_{t-1})\langle x_{t-1}-x^*,\,g_{t-1}\rangle
\le
(1-\beta_{t-1})\bigl(f_{S_{t-1}}(x^*)-f_{S_{t-1}}^*\bigr)
\le
f_{S_{t-1}}(x^*)-f_{S_{t-1}}^*.
\]
Therefore
\[
A_t
\le
\eta\|d_t\|^2 + \beta_{t-1}A_{t-1} + \bigl(f_{S_{t-1}}(x^*)-f_{S_{t-1}}^*\bigr).
\]
Since \(A_0=0\) (because \(d_0=0\)), unrolling the recursion gives
\[
A_t
\le
\sum_{j=1}^{t}
\left(\prod_{k=j}^{t-1}\beta_k\right)
\left[
\eta\|d_j\|^2 + f_{S_{j-1}}(x^*)-f_{S_{j-1}}^*
\right].
\]
Using \(\beta_k\le \beta_{\max}=1-\epsilon\), we obtain
\[
A_t
\le
\sum_{j=1}^{t}
(1-\epsilon)^{\,t-j}
\left[
\eta\|d_j\|^2 + f_{S_{j-1}}(x^*)-f_{S_{j-1}}^*
\right].
\]
Since \(0\le \beta_t\le 1\), we also have \(\beta_t A_t \le A_t\), and thus
\[
-\beta_t\langle x_t-x^*,\,d_t\rangle
\le
\sum_{j=1}^{t}
(1-\epsilon)^{\,t-j}
\left[
\eta\|d_j\|^2 + f_{S_{j-1}}(x^*)-f_{S_{j-1}}^*
\right].
\]
Taking expectation and using Lemma~\ref{lem:fixedcap_velocity} together with Assumption~\ref{ass:finite_gap},
\begin{align*}
\mathbb{E}\bigl[-\beta_t\langle x_t-x^*,\,d_t\rangle\bigr]
&\le
\sum_{j=1}^{t}
(1-\epsilon)^{\,t-j}
\left[
\eta\frac{G^2}{\epsilon} + \sigma_f^2
\right] \\
&\le
\frac{1}{\epsilon}
\left(
\eta\frac{G^2}{\epsilon} + \sigma_f^2
\right)
=
\frac{\eta G^2}{\epsilon^2} + \frac{\sigma_f^2}{\epsilon}.
\end{align*}
\end{proof}

\begin{theorem}[Convex convergence under a fixed cap]
\label{thm:fixed_cap_convex}
Assume each \(f_S\) is convex, Assumption~\ref{growth} holds, Assumption~\ref{ass:finite_gap} holds, and \(\mathbb{E}[g_t\mid x_t]=\nabla f(x_t)\). Consider the iterates generated by \eqref{eq:mom_for_proof} with a fixed cap \(0\le \beta_t\le \beta_{\max}<1\), and let \(\epsilon:=1-\beta_{\max}\). Define
\[
\bar x_T := \frac{1}{T}\sum_{t=0}^{T-1} x_t.
\]
Then
\[
\mathbb{E}\bigl[f(\bar x_T)-f(x^*)\bigr]
\le
\frac{\|x_0-x^*\|^2}{2\eta\epsilon T}
+
\eta G^2\,\frac{2+\epsilon}{2\epsilon^3}
+
\sigma_f^2\,\frac{1+\epsilon-\epsilon^2}{\epsilon^2}.
\]
In particular, with \(\eta=T^{-1/2}\),
\[
\mathbb{E}\bigl[f(\bar x_T)-f(x^*)\bigr]
\le
\frac{\|x_0-x^*\|^2}{2\epsilon\sqrt{T}}
+
\frac{G^2(2+\epsilon)}{2\epsilon^3\sqrt{T}}
+
\sigma_f^2\,\frac{1+\epsilon-\epsilon^2}{\epsilon^2},
\]
that is, the method approaches a \(\sigma_f^2\)-dependent neighborhood at rate \(O(1/\sqrt{T})\).
\end{theorem}

\begin{proof}
Let
\[
\Delta_t := \|x_t-x^*\|^2.
\]
Using the update \(x_{t+1}=x_t-\eta d_{t+1}\),
\[
\Delta_{t+1}
=
\Delta_t
-2\eta\langle x_t-x^*,\,d_{t+1}\rangle
+\eta^2\|d_{t+1}\|^2.
\]
Expanding \(d_{t+1}=\beta_t d_t+(1-\beta_t)g_t\),
\[
\Delta_{t+1}
=
\Delta_t
-2\eta\beta_t\langle x_t-x^*,\,d_t\rangle
-2\eta(1-\beta_t)\langle x_t-x^*,\,g_t\rangle
+\eta^2\|d_{t+1}\|^2.
\]
By convexity of \(f_{S_t}\),
\[
\langle x_t-x^*,\,g_t\rangle
\ge
f_{S_t}(x_t)-f_{S_t}(x^*).
\]
Therefore
\[
\Delta_{t+1}
\le
\Delta_t
-2\eta\beta_t\langle x_t-x^*,\,d_t\rangle
-2\eta(1-\beta_t)\bigl(f_{S_t}(x_t)-f_{S_t}(x^*)\bigr)
+\eta^2\|d_{t+1}\|^2.
\]

Now write
\[
f_{S_t}(x_t)-f_{S_t}(x^*)
=
\bigl(f_{S_t}(x_t)-f_{S_t}^*\bigr)
-
\bigl(f_{S_t}(x^*)-f_{S_t}^*\bigr).
\]
Since \(f_{S_t}(x_t)-f_{S_t}^*\ge 0\) and \(1-\beta_t\ge \epsilon\),
\[
-(1-\beta_t)\bigl(f_{S_t}(x_t)-f_{S_t}(x^*)\bigr)
\le
-\epsilon\bigl(f_{S_t}(x_t)-f_{S_t}(x^*)\bigr)
+
(1-\epsilon)\bigl(f_{S_t}(x^*)-f_{S_t}^*\bigr).
\]
Substituting this bound yields
\[
\Delta_{t+1}
\le
\Delta_t
-2\eta\epsilon\bigl(f_{S_t}(x_t)-f_{S_t}(x^*)\bigr)
+2\eta(1-\epsilon)\bigl(f_{S_t}(x^*)-f_{S_t}^*\bigr)
-2\eta\beta_t\langle x_t-x^*,\,d_t\rangle
+\eta^2\|d_{t+1}\|^2.
\]

Taking expectation and using \(\mathbb{E}[g_t\mid x_t]=\nabla f(x_t)\),
\[
\mathbb{E}\bigl[f_{S_t}(x_t)-f_{S_t}(x^*)\bigr]
=
\mathbb{E}\bigl[f(x_t)-f(x^*)\bigr].
\]
Applying Assumption~\ref{ass:finite_gap}, Lemma~\ref{lem:fixedcap_cross}, and Lemma~\ref{lem:fixedcap_velocity}, we get
\[
\mathbb{E}[\Delta_{t+1}]
\le
\mathbb{E}[\Delta_t]
-2\eta\epsilon\,\mathbb{E}\bigl[f(x_t)-f(x^*)\bigr]
+2\eta(1-\epsilon)\sigma_f^2
+2\eta\left(\frac{\eta G^2}{\epsilon^2}+\frac{\sigma_f^2}{\epsilon}\right)
+\eta^2\frac{G^2}{\epsilon}.
\]
Rearranging,
\[
2\eta\epsilon\,\mathbb{E}\bigl[f(x_t)-f(x^*)\bigr]
\le
\mathbb{E}[\Delta_t]-\mathbb{E}[\Delta_{t+1}]
+
\eta^2G^2\left(\frac{2}{\epsilon^2}+\frac{1}{\epsilon}\right)
+
2\eta\sigma_f^2\left((1-\epsilon)+\frac{1}{\epsilon}\right).
\]
Summing from \(t=0\) to \(T-1\) and telescoping,
\[
\sum_{t=0}^{T-1}\mathbb{E}\bigl[f(x_t)-f(x^*)\bigr]
\le
\frac{\|x_0-x^*\|^2}{2\eta\epsilon}
+
T\left[
\eta G^2\,\frac{2+\epsilon}{2\epsilon^3}
+
\sigma_f^2\,\frac{1+\epsilon-\epsilon^2}{\epsilon^2}
\right].
\]
Finally, by convexity of \(f\) and Jensen's inequality,
\[
\mathbb{E}\bigl[f(\bar x_T)-f(x^*)\bigr]
\le
\frac{1}{T}\sum_{t=0}^{T-1}\mathbb{E}\bigl[f(x_t)-f(x^*)\bigr],
\]
which proves the claim.
\end{proof}

The fixed-cap result differs from the main theorem in that a nonvanishing cap \(\beta_{\max}<1\) leaves a residual stochastic bias in the adaptive memory term, which appears through the mini-batch mismatch quantity \(\sigma_f^2\). As a result, the method approaches the solution set at rate \(O(1/\sqrt{T})\), but only up to a noise-dependent neighborhood whose radius is controlled by \(\sigma_f^2\). This is analogous to the usual behaviour of stochastic methods under persistent gradient noise: with a fixed amount of memory, one retains a nonzero bias floor unless the noise itself vanishes. In the common interpolating regime, however, \(x^*\) minimizes each mini-batch loss as well, so \(f_S(x^*)=f_S^*\) almost surely and hence \(\sigma_f^2=0\). In that case the neighborhood term disappears, and the bound reduces to exact \(O(1/\sqrt{T})\) convergence.

\newpage
\section{ALGORITHM DETAILS}
\label{sec:algo}
\vspace{-1em}
\begin{algorithm}[h]
\caption{AdamW: Adam with Decoupled Weight Decay}
\begin{algorithmic}[1]
\STATE \textbf{Initialize:} Learning rate $\eta$, weight decay $\mu$, $\beta_1, \beta_2 \in [0,1)$, $\epsilon>0$ 
\STATE Initialize parameters $x_0$, first moment $d_0=0$, second moment $v_0=0$
\FOR{each iteration $t=1,2,\dots,T$}
    \STATE Compute gradient: $g_t = \nabla f(x_{t})$
    \STATE Update biased first moment estimate: $m_t = \beta_1 m_{t-1} + (1 - \beta_1) g_t$
    \STATE Update biased second moment estimate: $v_t = \beta_2 v_{t-1} + (1 - \beta_2) g_t^2$
    \STATE Compute bias-corrected first moment: $\hat{m}_t = \frac{m_t}{1 - \beta_1^t}$
    \STATE Compute bias-corrected second moment: $\hat{v}_t = \frac{v_t}{1 - \beta_2^t}$
    \STATE Compute update direction: $\hat{d}_t = \frac{\hat{m}_t}{\sqrt{\hat{v}_t} + \epsilon}$
    \STATE Apply weight decay: $x_{t+1} = x_{t} - \eta (\hat{d}_t + \mu x_t)$
\ENDFOR
\STATE \textbf{Return:} Learned parameters $x_T$
\end{algorithmic}
\label{alg:AdamW}
\end{algorithm}

\vspace{-1.5em}
\begin{algorithm}[h]
\caption{AM-AdamW: Adaptive Memory - Adam with Decoupled Weight Decay}
\begin{algorithmic}[1]
\STATE \textbf{Initialize:} Learning rate $\eta$, weight decay $\mu$, $\beta_{1\max}, \beta_2 \in [0,1)$, $\epsilon>0$, $\lambda$ 
\STATE Initialize parameters $x_0$, first moment $m_0=0$, second moment $v_0=0$
\FOR{each iteration $t=1,2,\dots,T$}
    \STATE Compute gradient: $g_t = \nabla f(x_{t})$
    \STATE Update biased second moment estimate: $v_t = \beta_2 v_{t-1} + (1 - \beta_2) g_t^2$
    \STATE Compute bias-corrected second moment: $\hat{v}_t = \frac{v_t}{1 - \beta_2^t}$

    \STATE Compute Adam Preconditioner: $P_t = (1 - \beta_{1\max}\prod_t\beta_{1t}) \text{diag} \big( \epsilon + \sqrt{\hat{v}_t} \big)$

    \STATE Compute adaptive $\beta_{1t}$: $\beta_{1t} = \text{Clip}_{[0, \beta_{1\max}]} \Bigg( \frac{\frac{(1 + \mu \eta)(\hat{f}(x_t) - f(x_t))}{\eta} - \left\langle d_t - g_t, g_t + \lambda P_t d_t \right\rangle_{P_t^{-1}(\lambda P_t + I)^{-1}}-\mu \langle x_t^\top, d_t - g_t \rangle}{\|d_t - g_t\|^2_{P_t^{-1}(\lambda P_t + I)^{-1}}}
     \Bigg)$
    
    \STATE Update first moment estimate: $d_{t+1} = (\lambda P_t + I)^{-1}\left( \left(1-\beta_{1t}\right)g_{t} +    \left(\lambda P_t + \beta_{1t}I\right)d_{t} \right)$
    \IF{Proximal Weight Decay}
    \STATE Update Parameters: $x_{t+1} =\frac{1}{1+\mu\eta}\bigg(x_t -\eta P_t^{-1}d_{t+1}\bigg)$
    \ELSIF{Decoupled Weight Decay}
    \STATE Update Parameters: $x_{t+1} =(1-\mu\eta)x_t -\eta P_t^{-1}d_{t+1}$
    \ENDIF
\ENDFOR
\STATE \textbf{Return:} Learned parameters $x_T$
\end{algorithmic}
\label{alg:AdamW2}
\end{algorithm}

\subsection{Practical Considerations}
In Adam, the preconditioner depends slightly on the value of $\beta_{1t}$ due to the bias correction term. This dependence prevents closed-form computation of $\beta_{1t}$, but because the bias correction decays exponentially, it can be safely ignored in practice by using the approximation $\prod_{i=1}^{t-1} \beta_{1i}$, which has been shown to incur no practical drawbacks. Alternatively, the correction can be conservatively bounded by multiplying with the maximum value $\beta_{1\max}$. Moreover, for standard choices of learning rate and weight decay, both coupled and decoupled weight decay behave similarly. In our experiments, we chose to multiply the bias correction factor by $\beta_{1\max}$ and to use decoupled weight decay throughout.

\subsection{Stochastic AdamW}

We again replace \( f(x_{t-1}, s_{t}) - f(x_t, s_t) \) with its first-order approximation around \( x_t \), given by: 
\begin{align*}
    f(x_{t-1}, s_{t}) - f(x_t, s_t) &\approx \nabla f(x_t, s_t)^\top (x_{t-1} - x_t) \\
    &= \eta_{t-1}g_t^T(P_{t-1}^{-1}d_t + \mu x_t)\\
    &\approx \eta_{t-1}g_t^T(P_{t}^{-1}d_t + \mu x_t)\tag{Assuming $P_{t-1}^{-1}\approx P_{t}^{-1}$ }
\end{align*}

Where to avoid storing 2 preconditioners (2 second moment estimates), we approximate $P_{t-1}^{-1}$ with $P_{t}^{-1}$.

\subsection{Per-Layer AdamW}
Our method requires computing the preconditioner twice: once to determine the value of $\beta_1$ and once again to perform the parameter update. Although this overhead is negligible in large-scale experiments, it can be eliminated, and even lead to minor performance gains, by computing a separate $\beta_1$ for each layer. In this variant, the same preconditioner can be reused within each layer: first to compute $\beta_1$, and then to update that layer’s parameters. This approach not only improves efficiency but also allows different momentum parameters to be assigned to different layers. However, our empirical observations indicate that the momentum dynamics across layers are relatively similar. A comparison of our method with and without per-layer adjustment on the LLaMA 100M experiment is provided in Table~3.

\begin{table}[h]
    \centering
    \caption{Comparison of our method on LLaMA-100M with shared vs.\ per-layer momentum parameter $\beta_1$ and with or without the $P_{t-1}^{-1}\approx P_{t}^{-1}$ approximation . In all cases the performance is similar, for computational efficiency all experiments in the main paper were conducted using per-layer $\beta_1$ and $P_{t-1}^{-1}\approx P_{t}^{-1}$ }
    \label{tab:per_layer_vs_shared}
    \begin{tabular}{lcc}
        \toprule
        \textbf{Method} & $\mathbf{P_{t-1}^{-1}}$ & $\mathbf{P_{t-1}^{-1}\approx P_{t}^{-1}}$ \\
        \midrule
        Shared $\beta_1$  & 3.811±0.088 & 3.814±0.088 \\
        Per-layer $\beta_1$  & 3.792±0.085 & 3.792±0.084 \\
        \bottomrule
    \end{tabular}
\end{table}

\subsection{Computation Overhead and Convergence Speedup}
\label{sec:computations}

We compare the computational cost and convergence speedup of \textbf{AM-AdamW} relative to the standard AdamW optimizer across models of varying sizes. Table~\ref{tab:am_adamw_overhead} shows that the adaptive momentum mechanism introduces negligible overhead, below $0.5\%$ even for small models, while consistently accelerating convergence by up to $1.5\times$ for larger models.

\begin{table}[h]
\centering
\footnotesize
\caption{Average per-iteration runtime and relative convergence speedup of AM-AdamW compared to AdamW. Overhead is the relative increase in iteration time; speedup is the ratio of steps to reach the same validation loss.}
\setlength{\tabcolsep}{5.5pt}   % tighter columns
\renewcommand{\arraystretch}{1.25} % taller rows (incl. header)
\begin{tabular}{@{}lcccc@{}}
\toprule
\textbf{Model Size} &
\shortstack[c]{\textbf{Avg. Time per}\\ \textbf{Iteration (AdamW)}} &
\shortstack[c]{\textbf{Avg. Time per}\\ \textbf{Iteration (AM-AdamW)}} &
\shortstack[c]{\textbf{AM-AdamW}\\ \textbf{Overhead (\%)}} &
\shortstack[c]{\textbf{Convergence Speedup}\\ \textbf{ (AM-AdamW/AdamW)}} \\ \midrule
20M  & 1.21s  & 1.21s  & 0.31\% & 1.11$\times$ \\
60M  & 2.77s  & 2.78s  & 0.36\% & 1.41$\times$ \\
100M & 6.23s  & 6.26s  & 0.48\% & 1.45$\times$ \\
350M & 13.91s & 13.95s & 0.29\% & 1.52$\times$ \\
1B   & 78.89s & 78.98s & 0.11\% & 1.54$\times$ \\ \bottomrule
\end{tabular}
\label{tab:am_adamw_overhead}
\end{table}

\section{EXPERIMENTAL SETUP FOR SECTION~\ref{experiments}}
\label{sec:setup_experiment}

\subsection{Convex problems}
For the binary classification tasks, the learning rate was chosen based on an estimate of the smoothness of the feature matrix. The same procedure was applied to the multiclass classification problems. While this approach is not theoretically justified for all optimizers, it provided a consistent and practical basis for comparison. Importantly, for each dataset, the same learning rate was used across all optimizers to ensure fairness. For Adam, we fixed the learning rate to 0.001 for all experiments. We used the codebase from \href{https://github.com/konstmish/opt_methods}{https://github.com/konstmish/opt\_methods}

\subsection{Image Classification}
\paragraph{Dataset Augmentations}
\begin{itemize}
    \item \textbf{ConvNets for CIFAR10/100:} For CIFAR-10 and CIFAR-100, the training pipeline consists of random cropping to 32×32 with a padding of 4, random horizontal flipping and normalization using dataset-specific mean and standard deviation. The validation pipeline includes only normalization.
     \item \textbf{ResNet50 for ImageNet:} For ResNet-50, the training pipeline consists of random resized cropping to 224×224, random horizontal flipping and normalization using ImageNet mean and standard deviation. During validation, the transformations include resizing to 256 pixels, center cropping to 224×224, conversion to a tensor, and the same normalization applied in training.
\end{itemize}

\begin{table}[h]
\scriptsize
    \centering
    \begin{minipage}{0.48\linewidth}
        \centering
        \caption{Hyper-parameter settings for the CIFAR100 experiments.}
        \begin{tabular}{ll}
            \toprule
            \textbf{Hyper-parameter} & \textbf{Value} \\
            \midrule
            Architecture & ResNet50 and WRN-40-10 \\
            Epochs & 200 \\
            Batch size & 128 \\
            Optimizers & MGD \\
            LR schedule & cosine \\
            Weight decay & 0 \\
            Momentum for MGD & 0.9 \\
            Learning Rate & 0.1 \\
            AM-MGD $\lambda$ & 0.1 \\
            $\beta_{\max}$ & $0.9-0.1\lambda$\\
            \bottomrule
        \end{tabular}
        \label{tab:hyperparams_cifar100}
    \end{minipage}
    \hfill
    \begin{minipage}{0.48\linewidth}
        \centering
        \caption{Hyper-parameter settings for the CIFAR10 experiments.}
        \begin{tabular}{ll}
            \toprule
            \textbf{Hyper-parameter} & \textbf{Value} \\
            \midrule
            Architecture & ResNet18 and VGG-19 \\
            Epochs & 200 \\
            Batch size & 128 \\
            Optimizers & MGD \\
            LR schedule & step with r=0.1 at [100,150] \\
            Weight decay & 0 \\
            Momentum for MGD & 0.9 \\
            Learning Rate & 0.1 \\
            AM-MGD $\lambda$ & 0.1 \\
            $\beta_{\max}$ & $0.9-0.1\lambda$\\
            \bottomrule
        \end{tabular}
        \label{tab:hyperparams_cifar10}
    \end{minipage}
\end{table}

For the ablation studies, we used the same experimental settings, varying only the hyperparameter under investigation. Additionally, no learning rate schedulers were applied.

\begin{table}[h]
\scriptsize
    \centering
        \caption{Hyper-parameter settings for the ResNet18 trained on ImageNet experiments.}
        \begin{tabular}{ll}
        \toprule
        \textbf{Hyper-parameter} & \textbf{Value} \\
        \midrule
        Architecture & ResNet18 \\
        Epochs & 100 \\
        Batch size & 256 \\
        Optimizers & MGD \\
        LR schedule & step with r=0.1 at [30,60,90] \\
        Weight decay & 1e-4 \\
        Momentum for MGD & 0.9 \\
        Learning Rate & 1 \\
        AM-MGD $\lambda$ & 0.1 \\
        $\beta_{\max}$ & $0.9-0.1\lambda$\\
        \bottomrule
    \end{tabular}
\end{table}

\subsection{Pretraining Large Language Models}

\begin{table}[h]
\scriptsize
\centering
\begin{minipage}{0.48\linewidth}
    \centering
    \caption{Hyper-parameter settings for the LLaMA experiments.}
    \begin{tabular}{ll}
        \toprule
        \textbf{Hyper-parameter} & \textbf{Value} \\
        \midrule
        Learning Rate & 0.001 \\
        Weight Decay & 0.0001 \\
        Batch Size & 1024 \\
        Model Precision & BF16 \\
        Scheduler & Cosine with warm-up \\
        Warm-up Ratio & 10\% \\
        Grad Clipping & 1.0 \\
        $\beta_1$ & 0.9 \\
        $\beta_2$ & 0.999 \\
        $\epsilon$ & 1e-8 \\
        Seq-len & 1024 \\
        AM-AdamW $\lambda$ & 0.1 \\
        $\beta_{1\max}$ & $0.9 - 0.1\lambda$ \\
        \midrule
        \textbf{Eval Precision} & BF16 \\
        \textbf{Eval Seq-len} & 1024 \\
        \bottomrule
    \end{tabular}
    \label{tab:hyperparams_llm}
\end{minipage}
\hfill
\begin{minipage}{0.48\linewidth}
    \centering
    \caption{Architectural details for the LLaMA models.}
    \begin{tabular}{ll}
        \toprule
        \textbf{Component} & \textbf{Value} \\
        \midrule
        \# Parameters & 20M / 60M / 100M / 350M / 1B \\
        Hidden Size & 256 / 512 / 640 / 1024 / 2048 \\
        Intermediate Size & 688 / 1376 / 1708 / 2736 / 5461 \\
        Attention Heads & 4 / 8 / 10 / 16 / 32 \\
        Hidden Layers & 4 / 8 / 12 / 24 / 24 \\
        Activation Function & \texttt{silu} \\
        Normalization & RMSNorm ($\epsilon = 1\mathrm{e}{-6}$) \\
        Vocab Size & 32,000 \\
        Max Seq. Length & 1024 \\
        Initializer Range & 0.02 \\
        Model Type & \texttt{llama} \\
        Transformers Version & 4.28.1 \\
        \bottomrule
    \end{tabular}
    \label{tab:llama_architecture}
\end{minipage}
\end{table}

\newpage
\clearpage
\section{OMITTED EXPERIMENTAL RESULTS}

\subsection{Convex Problems}
\label{sec:conv}

\begin{figure}[h]
%\vspace{-0.1in}
\centering
\includegraphics[width=0.64\textwidth]{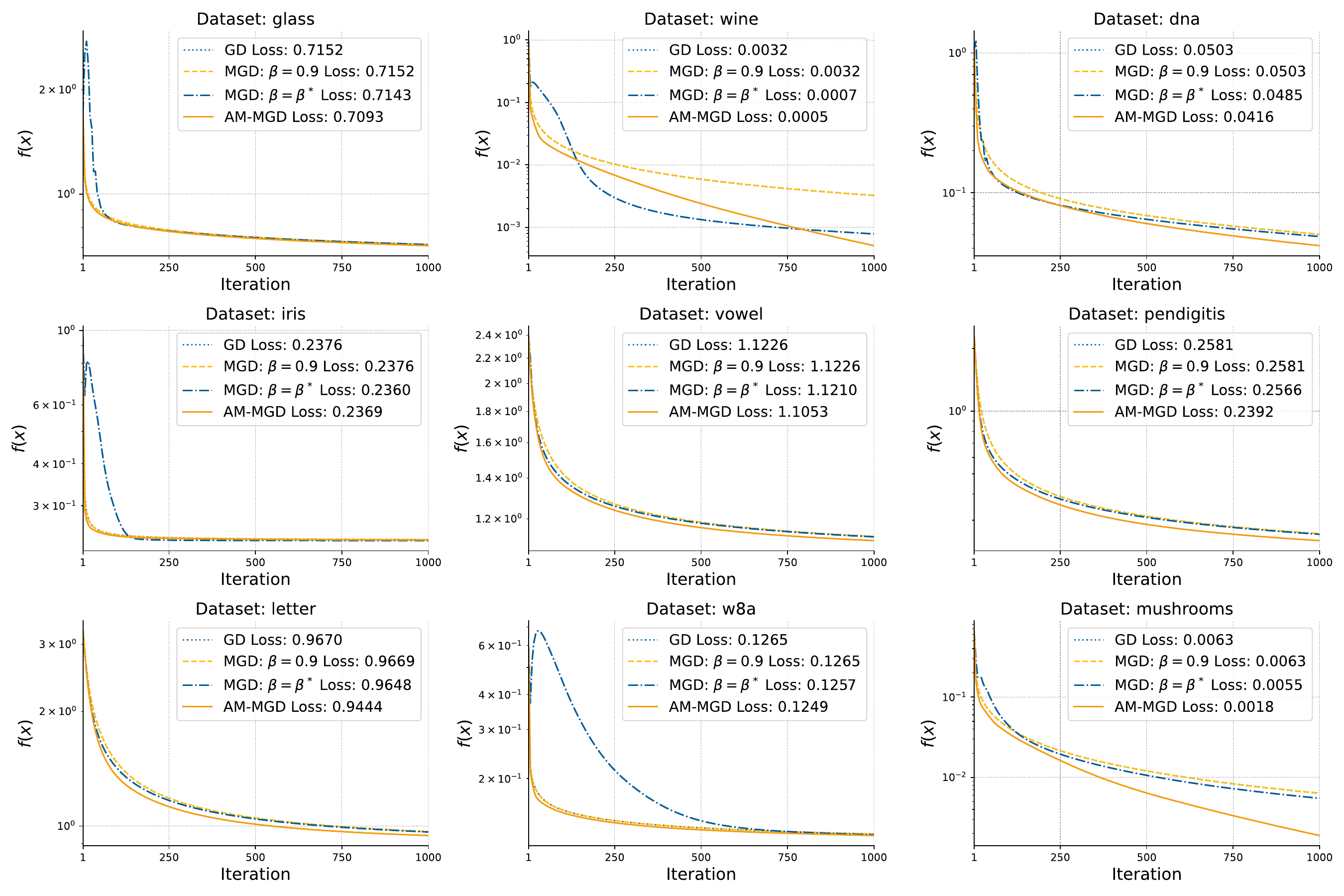} 
\caption{Logistic Loss over time, AM-MGD (red curve) outperforms fixed momentum on all but one experiments. GD and MGD with $\beta=0.9$ practically overlap. We run all algorithms for 10000 iterations with a learning rate $\eta=1/L$ where $L$ is the smoothness}
\label{fig:logreg_mgd}
\end{figure}

\begin{figure}[h]
%\vspace{-0.1in}
\centering
\includegraphics[width=0.64\textwidth]{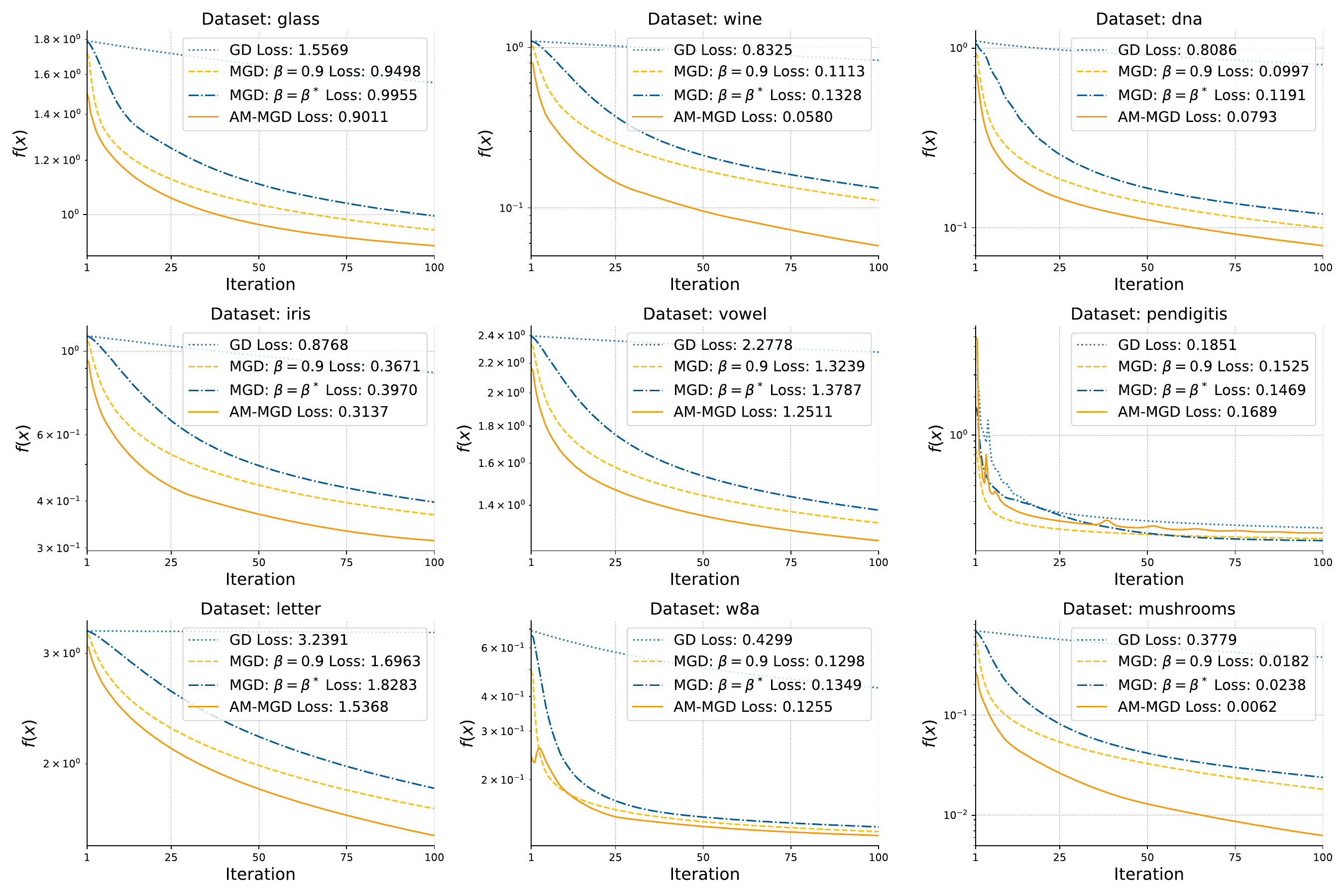} 
\caption{Logistic Loss over time, AM-Adam (purple curve) outperforms fixed momentum on all but one experiments. We run all algorithms for 1000 iterations with a learning rate of $\eta$=0.001}
\label{fig:logreg_adam}
\end{figure}

\newpage
\subsection{Detailed Curves for Figure~\ref{fig:ablations_twocol}c}
\label{sec:lr_ablation}

\begin{figure}[ht]
\centering
\includegraphics[width=\textwidth]{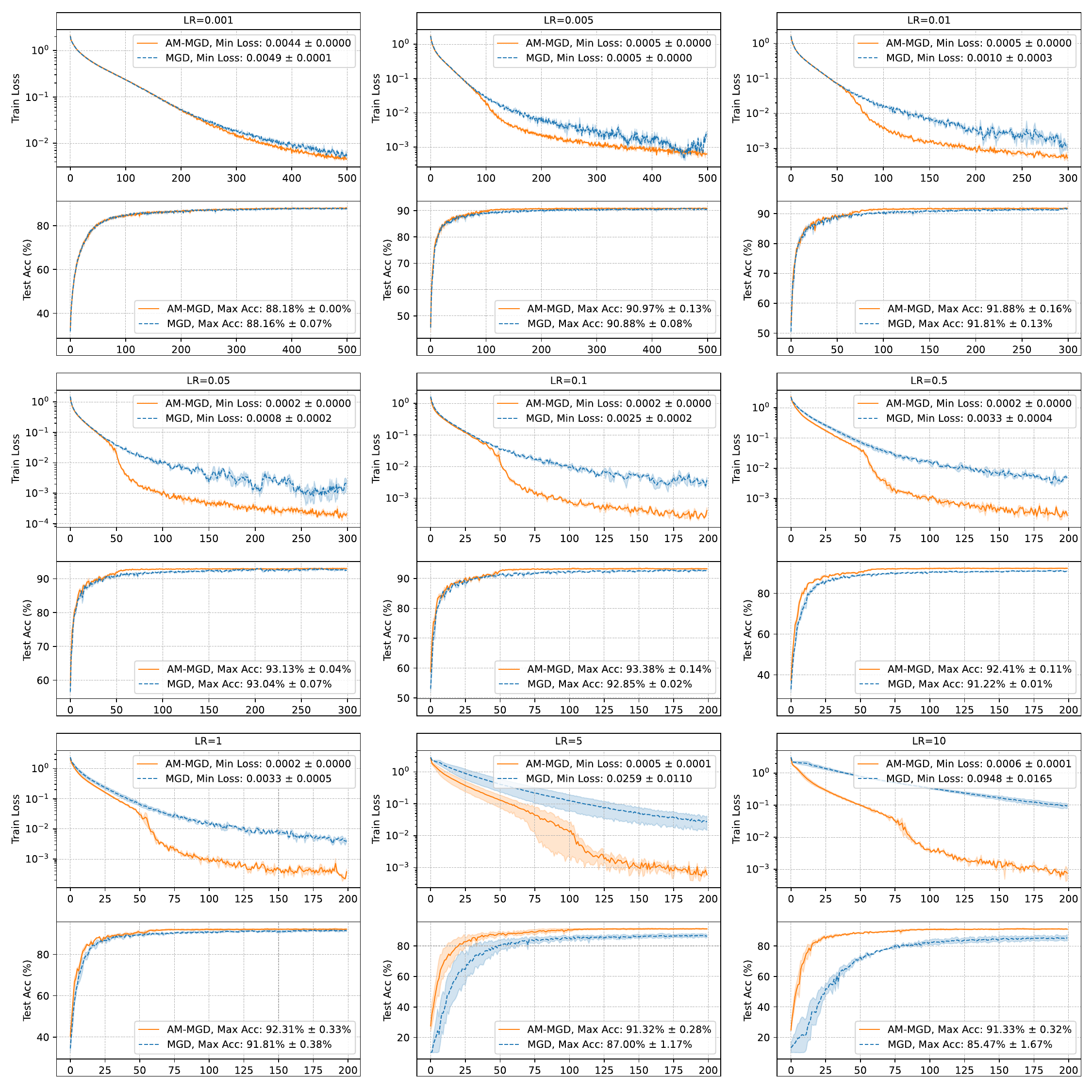}
\label{cifar10_manylr}
\caption{Training curves for different learning rates on ResNet18 trained on CIFAR10}
\end{figure}

\begin{figure}[h]
\centering
\includegraphics[width=\textwidth]{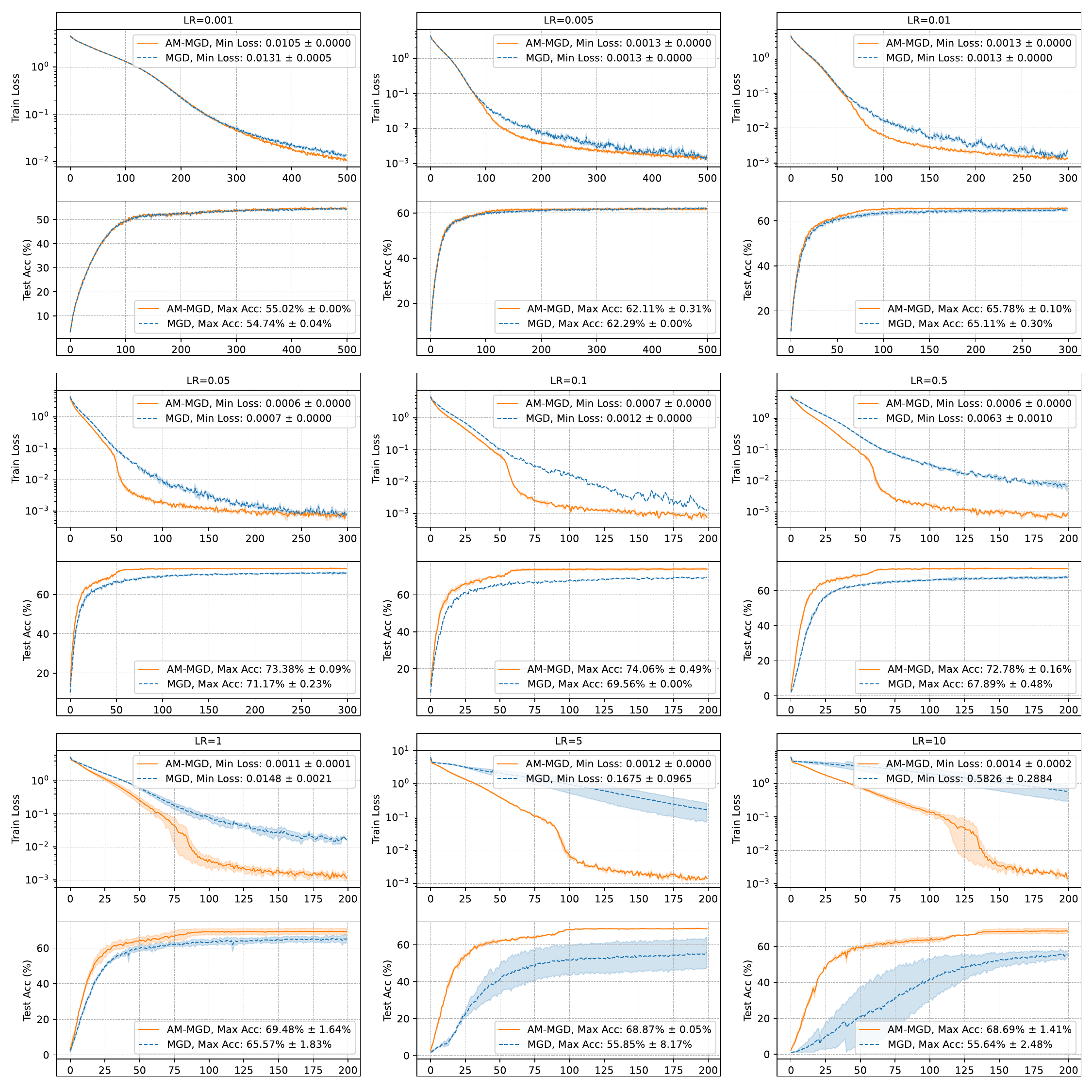} 
\label{cifar100_manylr}
\caption{Training curves for different learning rates on ResNet50 trained on CIFAR100}
\end{figure}

\clearpage
\newpage

\subsection{More plots on the behaviour of $\beta_t$}

\begin{figure}[ht]
\centering
\includegraphics[width=0.7\textwidth]{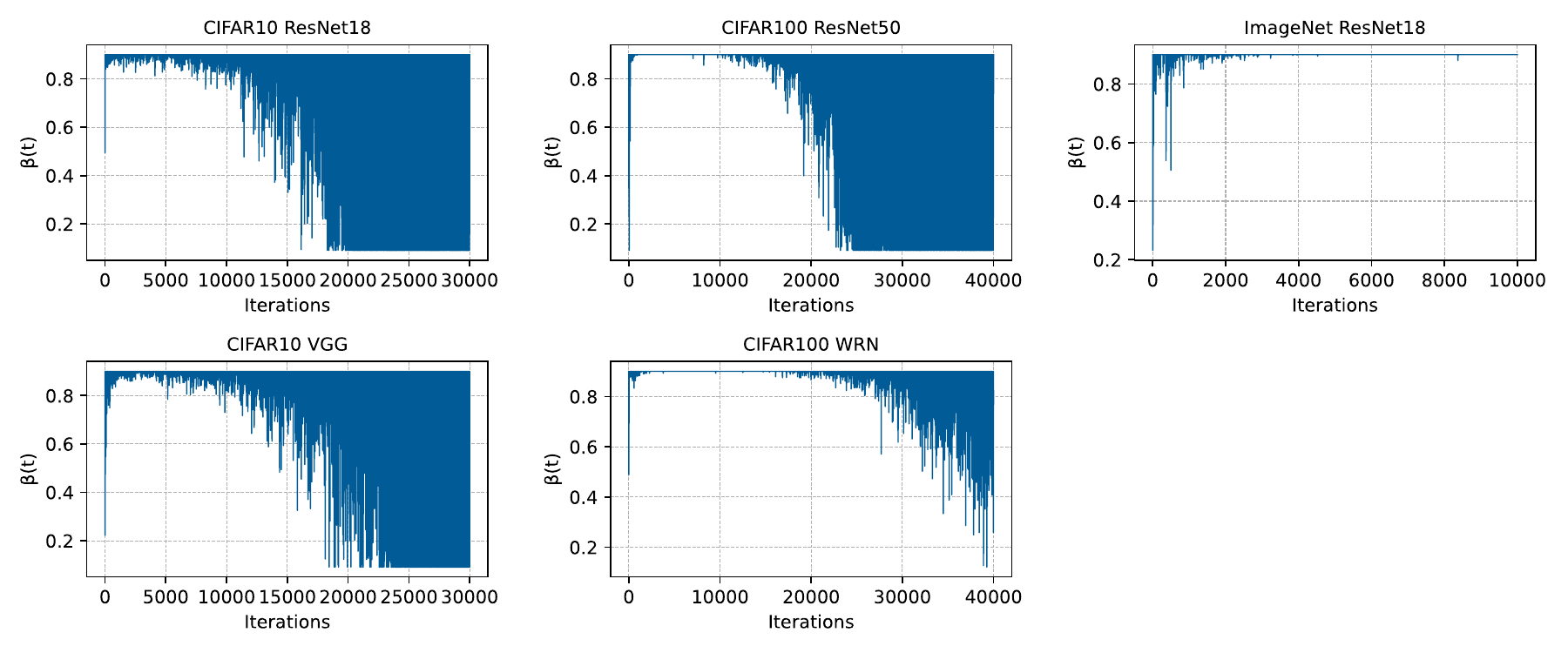} 
\caption{The adaptive $\beta_t$ across different experiments}
\end{figure}

\begin{figure}[h]
\centering
\includegraphics[width=0.7\textwidth]{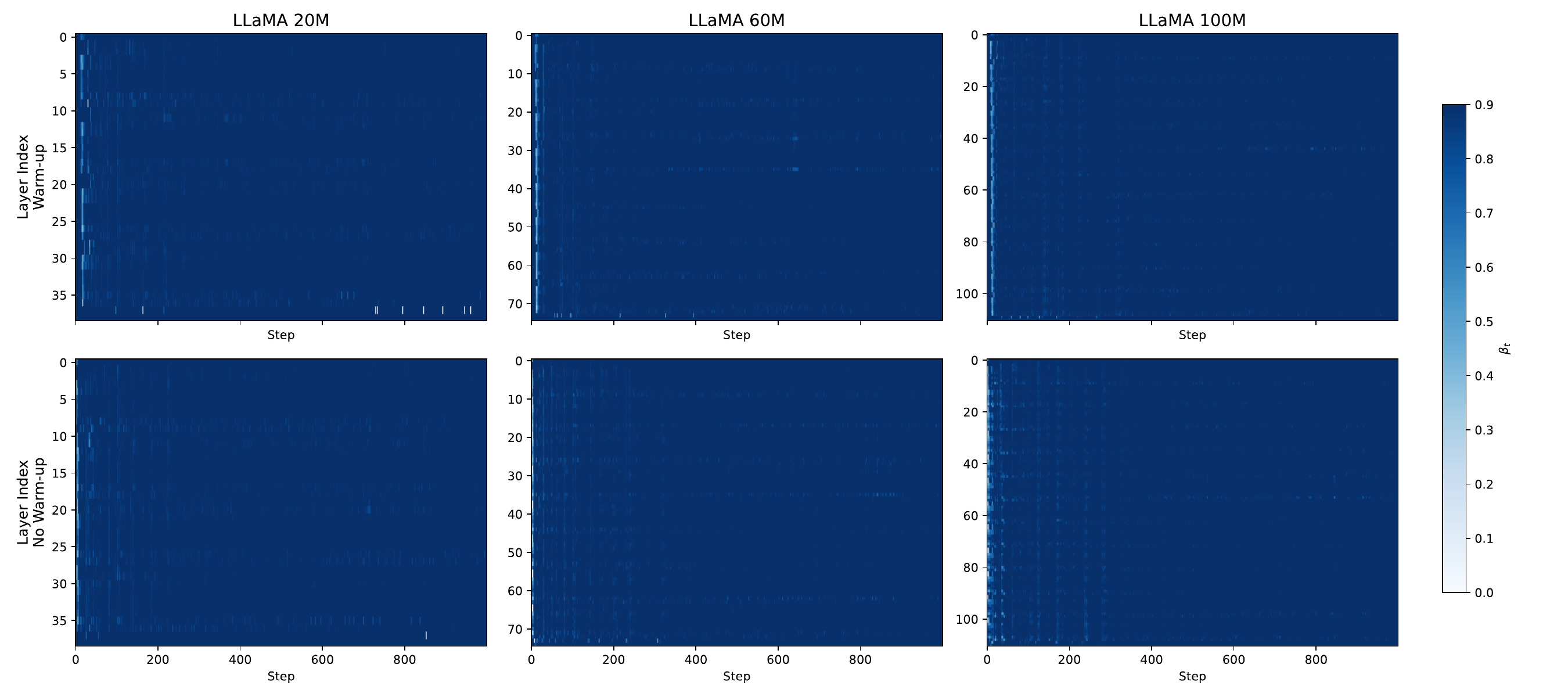} \caption{The layer-wise adaptive $\beta_t$ across different LLaMA experiments}
\end{figure}
\subsection{ImageNet Experiment}
\label{sec:imagenet}
\begin{figure}[h]
\centering
\includegraphics[width=0.8\textwidth]{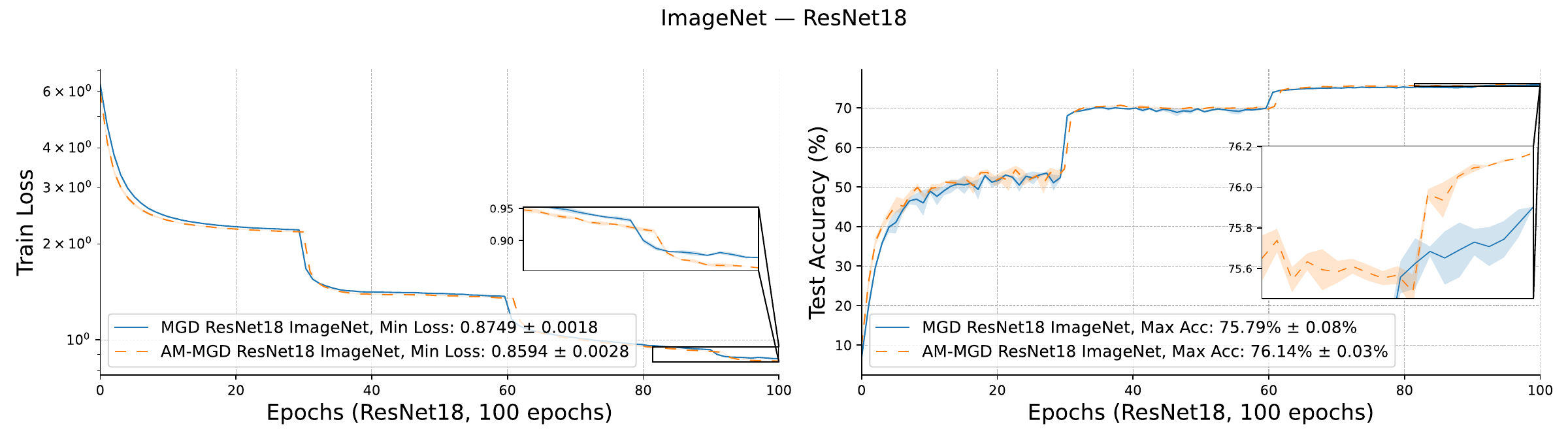} \caption{ResNet18 on ImageNet}
\end{figure}

\subsection{Comparison with other Optimizers}
We also compare our method against the Muon and Lion optimizers. We use the following publicly available implementations:
\begin{itemize}
    \item \textbf{Lion:} \href{https://github.com/huggingface/pytorch-image-models/blob/main/timm/optim/lion.py}{github.com/huggingface/pytorch-image-models}
    \item \textbf{Muon:} \href{https://github.com/KellerJordan/Muon/blob/master/muon.py}{github.com/KellerJordan/Muon}
\end{itemize}

For each optimizer, we perform a coarse hyperparameter grid search shown in Table~\ref{hs_search}. In our LLaMA experiments, we adopt the setup with warm-up enabled. For both Lion and Muon, the hyperparameter search is conducted only on the 20M model to assess the transferability of hyperparameters to larger model scales.

\begin{table}[h]
\small
    \centering
    \caption{Hyperparameter search ranges for Muon and Lion optimizers.}
    \label{tab:hyper_search_muon_lion}
    \begin{tabular}{lll}
        \toprule
        \textbf{Optimizer} & \textbf{Hyperparameter} & \textbf{Values Tested} \\
        \midrule
        Muon & Learning rate $\eta$                 & [1e-4, 5e-4, 1e-3, 5e-3, 1e-2] \\
         & Momentum coefficient $\beta$         & [0.9, 0.95, 0.99]           \\
        Lion & Learning rate $\eta$                 & [1e-4, 5e-4, 1e-3, 5e-3, 1e-2] \\
         & Second‐moment coefficient $\beta_2$  & [0.9, 0.95, 0.99, 0.999]          \\
        \bottomrule
        \label{hs_search}
    \end{tabular}
\end{table}

\begin{figure}[H]
    \centering

    \begin{minipage}{\linewidth}
        \centering
        \includegraphics[width=0.7\linewidth]{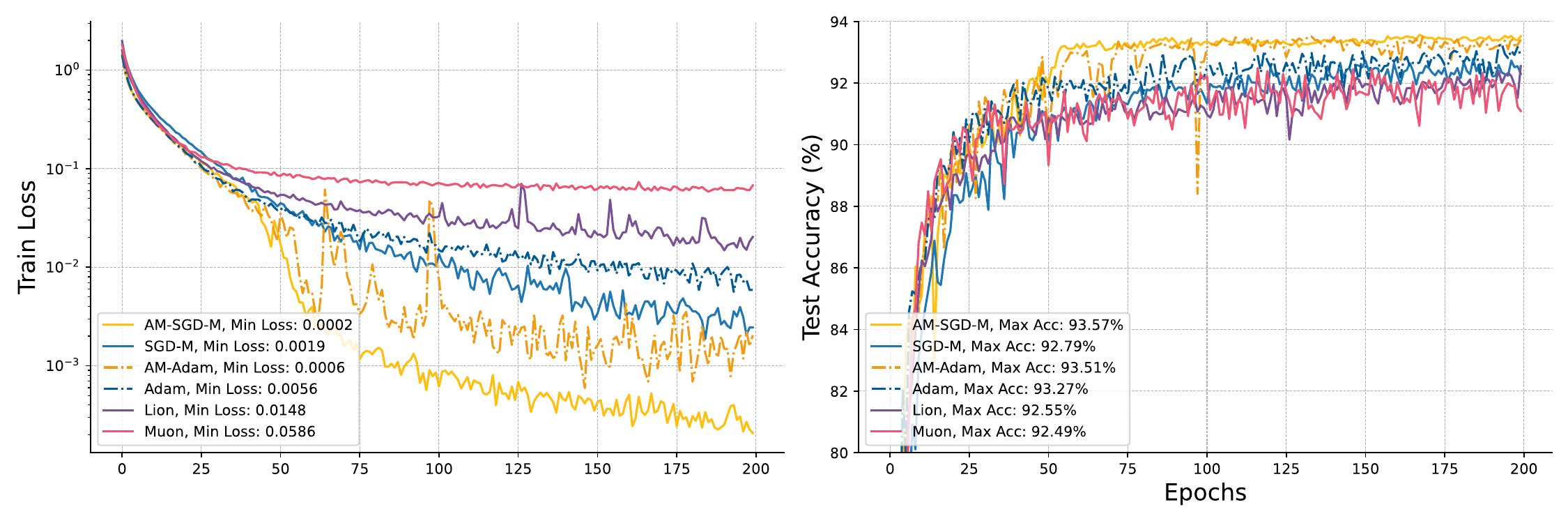}
        \caption*{(a) ResNet18 on CIFAR10}
    \end{minipage}

    \begin{minipage}{\linewidth}
        \centering
        \includegraphics[width=0.7\linewidth]{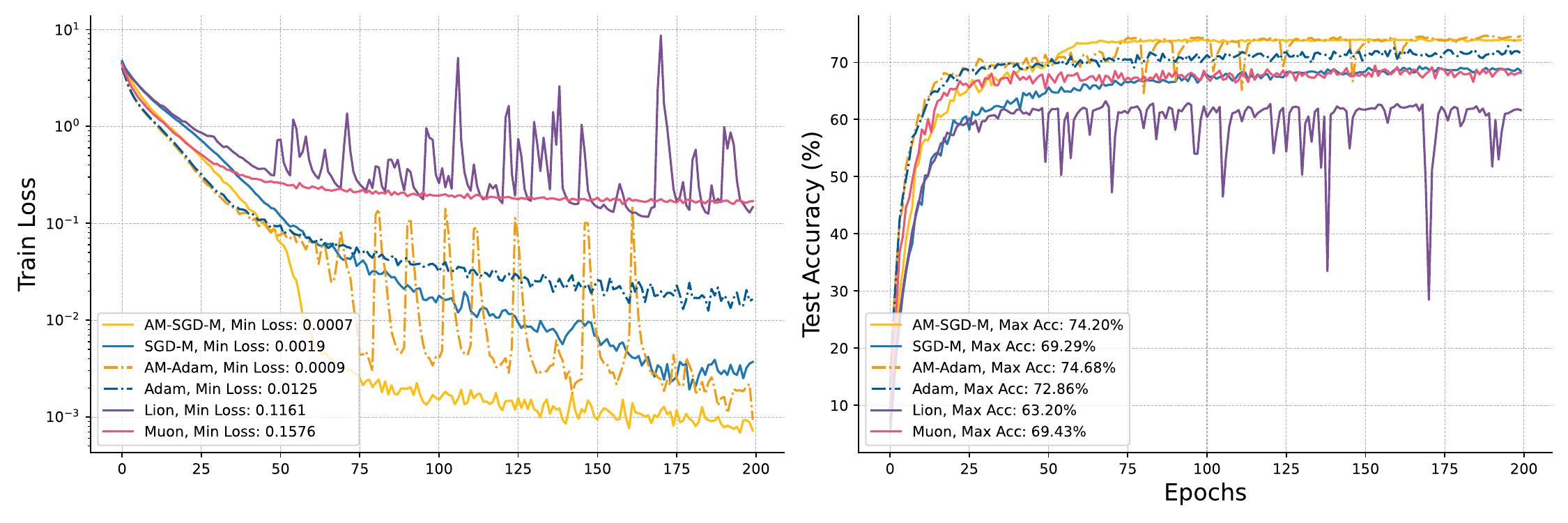}
        \caption*{(b) ResNet50 on CIFAR100}
    \end{minipage}

    \begin{minipage}{\linewidth}
        \centering
        \includegraphics[width=0.8\linewidth]{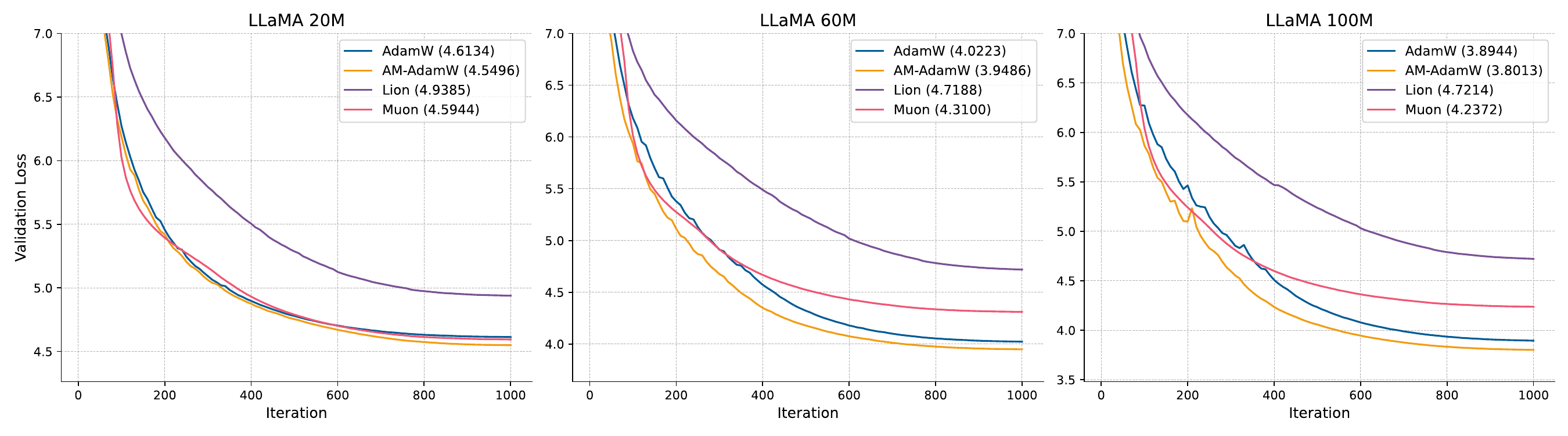}
        \caption*{(c) Three scales of LLaMA on C4}
    \end{minipage}

    \caption{Optimizer comparisons across different models and datasets.}
    \label{fig:combined_opt_comparison}
\end{figure}

\newpage
\clearpage

\end{document}